\documentclass[11pt, oneside]{article}   	
\usepackage{geometry}
\usepackage{hyperref} 
\usepackage{graphicx}				
\usepackage{amssymb}

\usepackage{amsmath}
\usepackage{mathtools}   
\usepackage{algorithm}
\usepackage{algpseudocode}

\usepackage{natbib}
 \usepackage{hyperref}
  \hypersetup{ colorlinks=true, linkcolor=blue, filecolor=magenta, urlcolor=cyan, pdftitle={Overleaf Example}, pdfpagemode=FullScreen, }
\usepackage{subcaption}
\usepackage{xcolor}
\usepackage{amsthm}
\usepackage{tcolorbox}
\usepackage[shortlabels]{enumitem}

\newtheorem{theorem}{Theorem}

\newcommand{\norm}[1]{\left\lVert#1\right\rVert}
\newcommand{\real}{\mathcal{R}}

\newcommand{\rulesep}{\unskip\ \vrule\ }

\title{Learning to Accelerate by the Methods of Step-size Planning}
\author{Hengshuai Yao}

\begin{document}

\author{Hengshuai Yao\\
    Huawei Technologies Canada\\
    Department of Computing Science\\
       University of Alberta, Edmonton, AB Canada
}


\maketitle

\begin{abstract}
Gradient descent is slow to converge for ill-conditioned problems and non-convex problems. An important technique for acceleration is step-size adaptation. The first part of this paper contains a structured review of step-size adaptation methods, including Polyak step-size, L4, LossGrad, Adam, IDBD, and Hypergradient descent, and the relation of step-size adaptation to meta-gradient methods which IDBD and Hypergradient descent belong to. 
We prove that meta-gradient for adapting a scalar step-size cannot be faster than a linear rate  $1-\frac{4\mu L}{(\mu +L)^2}$ for quadratic optimization where $u, L$ are the smallest and largest eigenvalues, and the approach suffers from ill conditioning, just like gradient descent. These are the first understandings of the convergence rate of meta-gradient step-size adaptation to our knowledge.   

In the second part of this paper, we propose a new class of methods of accelerating gradient descent that are distinctive from existing techniques. The new methods, which we call {\em step-size planning}, use the {\bfseries gradient descent update experience} to learn an improved way of updating the parameters. The methods organize the experience into $K$ steps away from each other to facilitate planning. From the past experience, our planning algorithm, Csawg, learns a step-size model which is a form of multi-step machine that predicts future updates. We extends Csawg to applying step-size planning multiple steps, which leads to further acceleration. 

One of the key components in Csawg is the diagonal-matrix form of step-size or the {\bfseries diagonal step-size} for short. 
We illustrate and discuss that the diagonal step-size has the same projection power as the full-sized matrix (counter-intuitive as the linear transformation theory tells us), and it is highly scalable for future large scale applications. We discuss and prove the oddness to allow and use {\bfseries negative step-sizes} in the diagonal step-size which surprisingly makes sense in both deterministic and stochastic gradient descent. 
We develop the first formal understanding of the necessity of using the diagonal step-size for faster convergence, which shows that the widely used scalar step-size for gradient descent leads to slow convergence because it is {\bfseries over-stretched} in that it tries to balance between the smallest and the largest optimal step-sizes for individual parameters. 

We show for a convex problem, our methods surpass the convergence rate of Nesterov's accelerated gradient, $1 - \sqrt{\frac{\mu}{L}}$, where $\mu, L$ are the strongly convex factor of the loss function $F$ and the Lipschitz constant of $F'$, which is the theoretical limit for the convergence rate of first-order methods. On the well-known non-convex Rosenbrock function, our planning methods achieve zero error below 500 gradient evaluations, while gradient descent takes about 10000 gradient evaluations to reach a $10^{-3}$ accuracy. 
We observed that the acceleration over gradient descent and even Nesterov's accelerated gradient appears to be several orders faster. 
We discuss the connection of step-size planing to planning in reinforcement learning, in particular, Dyna architectures. 
\end{abstract}

\section{Introduction}
There are lots of experience in many applications of artificial intelligence processed by computer programs to build certain systems and give instructions for future scenarios to improve our intelligence and life. To name some of today's most notable examples of using experience, a tree search program trains a superhuman-level player by vigorously playing the game against itself \citep{silver2016mastering};
\citep{mnih2015human}'s work demonstrated for the first time a single computer program can play competitively a variety of Atari games \citep{atari_marc} that are even difficult for humans, without changing in codes, but simply training the program with reinforcement learning and neural networks using screen image pixels and scores from individual games;
GPT-3, a large neural network with 175 billion parameters \citep{gpt3}, which performs various nature language problems including theme identification, sentiment analysis, text summarization used in over 300 Apps, was trained with almost all available data from the Internet;
Autonomous driving makes uses of tons of human labeled data as experience to improve pedestrian detection \citep{dollar2011pedestrian} and collision avoidance \citep{levinson2011towards};
Clinical data, containing valuable experience of medical decisions and patient response, which is extremely expensive to obtain, is used for clinical trials and quantifying evidences
\citep{shortreed2011informing}; etc.

Yes. We use lots of experience from  the real world, computer games and simulations, and they do good. 
In this paper, we explore another kind of experience that is not from the above sources, but ubiquitous in deep learning systems, in particular, the  {\em experience of stochastic gradient descent (SGD) updates}.
The SGD algorithm and its variants are extensively used in training deep neural networks, e.g., temporal difference methods \citep{td,tdgammon} as used in AlphaGo, and Adam optimizer \citep{kingma2017adam} as used in GPT-3. 
For large applications like these examples, it takes numerous SGD updates to train a good network model and trillions of FLOPs with dedicated resources such as cpus and gpus. 
The SGD update experience matters because it can improve the quality of neural network models, but unfortunately it is not noticed by the literature yet. 
In this paper, we are going to show how the SGD update experience can be used to accelerate the convergence of SGD to find solutions faster. We are focused on presenting the ideas, illustrating how the algorithms works as well as understanding their theoretical properties, which have a general interest from optimization and %
machine learning communities.

There are many situations that gradient descent can be slow, and step-sizing is a fundamental and complex matter. 
In deep learning, this is extremely important given that SGD powers the training of deep neural networks. In the first place,  
why do we need to adapt the step-size of SGD update at all? A step-size that is too small causes slow learning for gradient descent. A  step-size that is too large is no good either because it causes knowledge already learned to be quickly overwritten. When should we learn fast and when should we learn slow is an important part of learning. For deep learning, the importance of step-size can be reflected by \citet{yoshua_practical}: The initial learning rate, which is the parameter for the step-size, ``is often the single most important hyper-parameter and one should always make sure that it has been tuned''.
Studying step-size adaptation dates back to at least as early as 1960s \citep{Rastrigin1963,schumer1968adaptive} when people started to looking at this problem and solved it using random search though gradient descent was found much earlier by Baron Augustin-Louis Cauchy and Jacques Hadamard. Haskell Curry gave a review of the history of gradient descent as well as applications of using gradient descent for solving some interesting nonlinear optimization problems (including fire control design). It appears that gradient descent makes its name in history not by accident: since \citet{curry1944method}'s analysis of convergence to stationary points of gradient descent, people already knew that ``the (convergence) argument is, incidentally, capable of generalization to certain cases where there are infinitely many parameters.'' \citet{curry1944method} also noted that the step-size can be found by trial and search from a large value and gradually halving it, which is still widely practised in using gradient descent algorithms nowadays. 
The problem of step-size adaptation popularized especially due to the applications of gradient descent in training multiple perceptrons \citep{Rumelhart:1986we,jacobs1988increased,sutton1992adapting,mathews1993stochastic,ostermeier1994step,aboulnasr1997robust,lee2015variable}. It certainly revived recently because in deep learning gradient descent algorithm are used to train extremely large neural networks \citep{kingma2017adam,adadelta,adagrad}. Though this paper is focused on the convergence rate of SGD, which is the training perspective, it should be noted that step-size also influences the generalization of the converged solution \citep{sharp_minima,sharp_minima2,minima_valley}.

In the following section, we review algorithms of adaptive step-size for gradient descent. 
Some of the reviewed techniques are from classical optimization and neural networks; while others are new advancements in deep learning. In Section \ref{sec:step-size-planing}, a basic form of our method, which is a one-time step-size planning method is presented. 
Section \ref{sec:relation_rl_planning} discusses the data view of gradient descent update and the relation of step-size planning to the planning problem in reinforcement learning. Section \ref{sec:exp} presents empirical results.
Section \ref{sec:multi-step} extends step-size planning repeatedly. 
Then in Section \ref{sec:conclusion} we conclude the paper and discuss interesting topics for future research.

\section{Background}\label{sec:review}
 There are in general two lines of developing step-size adaptation techniques in literatures. 
 One line is by solving an optimization function of the step-size, usually minimizing a residual that is normally expressed for the loss function or the distance of the current weight \footnote{We use the word {\em weight} to refer to the parameter vector.} to its convergence point, which we call the {\em functional approach}.  
 
 The other line is developed in close relation to momentum in gradient descent and widely adopted in neural networks, which we call the {\em momentum approach}. The idea of momentum step-sizing is that you want to step forward fast when you're on the right track, especially when the update is in a good direction towards the minimum. 
Intuitively, momentum is one good way of quantifying if the update direction is good.

Without loss of generality, we consider a loss function, $f(w)$, with one minimum being denoted by $w^* = \arg\min_w f(w), w \in \real^d$. 
Our reviews of step-sizes assume $f$ is deterministic and convex, and $w^*$ is unique, for the ease of discussion. 
In addition to step-adaptation methods, we also discuss the relationship of step-size adaptation to meta-gradient, which gains wide popularity in the recent meta-learning literature.
Any review of this kind is bounded by the authors' knowledge and readings, especially the literature of step-size adaptation and meta-gradient is large, and we apologize for the limited coverage. 

\subsection{Convergence Rates}\label{sec:conv_rate}
In this section, we briefly review some key terminology that are frequently used when we talk about the convergence rate of gradient descent algorithms. 

The L-smoothness of a function is important to establish a fast convergence rate for gradient descent algorithms.
 A function $f$ is called $L$-smooth if there exists a constant $L>0$ such that
\[
\| f'(x) - f'(y)\| \le L\|x- y \|, \quad \forall x, y \in \real^d.
\]

The Polyak-\L{}ojasiewicz (PL) condition for $f$ means the increase of $f(x)$ from the lowest point, $f(x^*)$, is upper bounded by the squared norm of the gradient of $f(x)$.   
Formally, there exists $\mu>0$ such that
\[
f(x) - f(w^*) \le \frac{1}{2\mu} \|f'(x) \|^2, \quad \forall x\in \real^d.
\]

Strongly convex implies the PL condition. 
A function $f$ is called $\mu$-strongly convex if
\[
f(x) \ge f(y) + f'(y)^T(x-y) +\frac{\mu}{2}\|x-y \|^2,  \quad \forall x, \forall y \in \real^d.
\]

The well-known convergence rates and algorithms given these conditions on $f$ are as follows. 
\begin{itemize}
\item 
If $f$ is L-smooth, gradient descent with a constant step-size $\alpha=1/L$ is guaranteed to find a $j \le k$, such that $\|f'(w_j) \| \le 1/\sqrt{k}$ in the first $k$ iterations, e.g., see \citep{course_gd_rate}.

\item
If $f$ is L-smooth and convex, gradient descent using the same constant $\alpha$ at the $k$th iteration converges at a rate of $1/k$, with $f(w_k) -f(w^*) \le \frac{L}{2k} \norm{w_0-w^*}^2$.  
In this case, Nesterov's accelerated gradient (AGD) converges at a rate of $1/k^2$ for the $k$th iteration, e.g, see \citep{course_gd_rate_bubeck}. AGD is a momentum method. After a normal gradient descent step, $w_{k+1} = x_k - \frac{1}{L} f'(x_k)$, the next $x$ is generated by adding the momentum from $w$: 
\[
x_{k+1} = w_{k+1} + \gamma_k (w_{k+1} - w_{k}).
\]
\citet{nesterov1983method} showed that if we  
schedule the momentum rate $\gamma_k$ in a special way, then we have $f(w_k) - f(w^*) \le O(2L/k^2)$ at iteration $k$. 

\item 
If $f$ is $L$-smooth and it also satisfies the Polyak-Lojasiewicz (PL) condition, 
then one can show that the gradient descent with $\alpha=1/L$ guarantees a strict error deduction from the previous iteration:
\[
f(w_k) -f(w^*) \le \left(1-\frac{\mu}{L}\right) \left(f(w_{k-1}) -f(w^*)\right), 
\]
which is a linear convergence rate. So for $f$ that is $L$-smooth and $\mu$-strongly convex, gradient descent has a linear rate $1-\mu/L$. In this case, AGD has a faster rate too \citep{nesterov2003introductory}, in particular:
\[
f(w_k) -f(w^*) \le \left(1-\sqrt{\frac{\mu}{L}}\right)^k \left(f(w_{0}) -f(w^*) + C_0 \|w_0-w^* \|^2 \right),
\]
where $C_0$ is a constant.
Note this can be achieved by a constant scheduling of $\gamma_k$ in addition to a recurrent one. The scheduling requires the knowledge of $\mu$ and $L$ with $\gamma_k=(\sqrt{L} - \sqrt{\mu})/(\sqrt{L} + \sqrt{\mu})$. \citet{nesterov2003introductory} also constructed a function $f$ such that the convergence rate cannot be faster than the linear rate $1-\sqrt{\frac{\mu}{L}}$. Thus this convergence rate bound is tight. However, this result only applies to first-order methods. Our paper aims to explore if it's possible to be faster than this long considered optimal rate by using only the gradient information, but the induced methods are not first order.

\end{itemize}

\subsection{Polyak Step-size}
Polyak step-size is a functional approach, and is derived by minimizing an upper bound function for $\norm{w_{k+1}(\alpha) - w^*}^2$. 

To see this, let's start with a sub-gradient method. 
At iteration $k+1$, the weight update is
\[
w_{k+1} = w_k -\alpha_k g_k,
\]
where $g_k$ is a sub-gradient of $f$ at $w_k$. The step-size is given by 
\begin{equation}\label{eq:polyak}
\alpha_k =  \frac{f(w_k) - f(w^*)}{ || g_k ||^2}.
\end{equation}
It achieves a convergence rate of $O(1/\sqrt{k})$ when $f$ is convex. One may think this is not an impressive rate, but note that $f$ does not have to be L-smooth. 

The Polyak step-size can be derived by expressing the new distance after the update to the minimum, $w^*$, using the old distance, 
\begin{align}
||w_{k+1} - w^* ||^2 & = || w_k -\alpha_k g_k -w^* ||^2,  \nonumber \\
& = || w_k - w^* ||^2 + \alpha_k^2 || g_k||^2- 2 \alpha_k (w_k - w^*)^Tg_k. \label{eq:polyak1}
\end{align}
Note the last term is a correlation between the old error vector and the current sub-gradient. First we have to make sure this deduction is non-negative otherwise the new distance grows larger. 
In fact, indeed $(w_k - w^*)^Tg_k$ is non-negative because 
\begin{equation}\label{eq:polyak2}
(w_k - w^*)^Tg_k  \ge f(w_k) - f(w^*) \ge 0,  
\end{equation}
which comes from that $g$ is the {\em sub-gradient} (thus requiring convex on $f$):
\[
f(w^*) \ge f(w_k) + (w^* - w_k )^Tg_k. 
\]
Good. We know the reduction is non-negative now. 
Next, the derivation switches to $f(w_k) - f(w^*)$ from $w_k - w^*$. 
Combining equations \ref{eq:polyak1} and \ref{eq:polyak2}, we have 
\[
||w_{k+1} - w^* ||^2  \le || w_k - w^* ||^2 + \alpha_k^2 || g_k||^2 - 2 \alpha_k[f(w_k) - f(w^*)].   
\]
Now we can minimize the upper bound of $||w_{k+1} - w^* || $, which is on the right-hand side of the equation, by a choice of step-size. 
First we can write $||w_{k+1} - w^* || $ in terms of error reduction:
\begin{align*}
||w_{k+1} - w^* ||^2  \le || w_k - w^* ||^2 - \alpha_k  \left[ 2 (f(w_k) - f(w^*)) - \alpha_k || g_k||^2 \right]. 
\end{align*}
To maximize the reduction, we solve the optimization problem:
\begin{align*}
\max_\alpha E(\alpha) & = \max_\alpha   \left[ 2 \alpha (f(w_k) - f(w^*)) - \alpha^2 || g_k||^2 \right] \\
& =  || g_k||^2 \max_\alpha   \left[ 2 \alpha \frac{f(w_k) - f(w^*)}{|| g_k||^2} - \alpha^2  \right], 
\end{align*}
giving
$
\alpha_k = \arg\max E(\alpha)
$, 
which is exactly the Polyak step-size in equation \ref{eq:polyak}. 
The convergence rate of $O(1/\sqrt{k})$ can be established using telescoping sum and the boundedness of the sub-gradient \citep{polyak_book}. 
The SPS steps-size extends this step-size to stochastic gradient descent \citep{sps_polyak}.

\subsection{L4}

L4 by \cite{rolinek2018l4} is a functional approach for step-size. It is based on 
\[
w_{k+1} = w_k - \alpha v_k, 
\]
where the update $v_k$ can be a sub-gradient, or provided by another algorithm. In their paper, they used gradient with momentum and ADAM to provide $v$. 
L4 is derived by minimizing the post-update loss: $f(w_{k+1}(\alpha))$. 
The post-update loss was approximated using $f(w_k)$, which is reasonable because gradient descent changes the weight only a little between successive updates. 
This done by a linearization of $f$:
\begin{equation}\label{eq:L4}
f(w_k - \alpha v_k ) \approx f(w_k) - \alpha_k f'(w_k)^T v_k.
\end{equation}
To approximately reach $f(w^*)$ after the update,  the right side is matched to $f(w^*)$. Equivalently, $\alpha_k = \arg\min_{\alpha}| f(w^*) - (f(w_k) - \alpha f'(w_k)^T v_k)|$, from which we derive that, \footnote{To account for numerical stability, a small perturbation is added to the denominator in the L4 step-size.}
\[
\alpha_k=\frac{f(w_k) - f(w^*)}{f'(w_k)^Tv_k}. 
\]
Interestingly, this matches exactly the Polyak step-size when $v_k$ is taken to be the current gradient for the convex case. 
The two step-sizes are developed from different perspectives. 
Polyak step-size is derived from an upper bound on $\norm{w_{k+1} -w^*}$, while L4 is from an approximation of $f(w_{k+1}) - f(w^*)$. 
Yet the two step-sizes match. 
The bridge is the convex of $f$.
Without $f$ being convex, the transition from weight difference to functional difference would not be possible. 
Thus the boundary of the two works is that L4 enables the deduction of the loss in the sense of the first order approximation; and in the case of $f$ is convex, the step-size also guarantees a maximum reduction of the weight difference to $w^*$.  

L4 authors proposed to adapt an estimate of $f(w^*)$ as learning proceeds. The basic technique is to set it to the lowest loss seen so far on mini-batches. Note that some mini-batches could have a very small loss. Imagine there is a mini-batch that is sampled more often than others, especially in an online setting like experience replay in deep reinforcement learning, this mini-batch is going to have much smaller loss than others; or this could happen just because of random initialization. Thus it appears that the minimal loss seen so far on mini-batches is not going to be achieved on the whole data. A remedy used in L4 is to increase it every iteration by a factor of $1+ 1/\gamma$, where $\gamma $ was set to $0.9$ in their experiments. 
However, this mechanism of adapting $f(w^*)$ suffers from high variances due to mini-batch sampling. Thus the step-size computed can be very unstable in practice. This was confirmed in an experiment by \cite{wojcik2019lossgrad}, in which they found L4 performed unstably across datasets with gradient descent providing the update direction $v$ (i.e., Polyak step-size).

\subsection{LossGrad}
LossGrad by  \cite{wojcik2019lossgrad} is a functional approach too. 
It also minimizes $f(w_{k+1}(\alpha))$ like L4, and extends the step-size function from linear in L4 to quadratic. 

Let us consider in L4, $v$ is set to the gradient. 
The basic idea of LossGrad is to start with an examination of the linearization error of L4, taken to be the difference between the two sides of equation \ref{eq:L4}:
\[
 e_k = {f\left(w_k - \alpha_k f'(w_k)\right) - \left[f(w_k) - \alpha_k \norm{f'(w_k)}^2\right]}, 
\]
Suppose the step-size of LossGrad is denoted by $t$. 
Now LossGrad uses the following function of step-size: 
\begin{align*}
f\left(w_k - tf'(w_k)\right) &\approx f(w_k) - t\norm{f'(w_k)}^2 +  e_k \left(\frac{t}{\alpha_k}\right)^2,  \\
&= f(w_k) - t\norm{f'(w_k)}^2 +  \left(\frac{e_k }{\alpha_k^2}\right)t^2 \\
&\stackrel{def}{=}q(t),
\end{align*}
wherein there is a linear part (same as in L4) plus a quadratic term. 
To make sense of this, note the left side can be any nonlinear but smooth function of step-size $t$;
and the right-hand side is a quadratic function of $t$. 
When $t = \alpha_k$, the above holds with equality and there is no approximation error by the quadratic. 
When $t=0$, it also holds with equality. 
When $t\in (0, \alpha_k)$, there is approximation error.
Due to the smoothness of $f\left(w_k - tf'(w_k)\right)$, one can expect this is a good approximation. 

Now we are ready to look what this quadratic approximation suggests how we update the improved step-size $t$ from $\alpha_k$. 
A positive $e_k$ indicates the linearization using $\alpha_k$ is too small and hence we should decrease $\alpha_k$;
while a negative $e_k$ indicates that linearization is too big and we should increase $\alpha_k$. 
In fact, in the LossGrad paper, increasing of the step-size is more careful than this. 
In particular, it examines the $q'(t)$ at $t = \alpha_k$. If the derivative is negative, it then increases $\alpha_k$; otherwise, it decreases $\alpha_k$. 
One can show that this decision boundary is, $\frac{e_k}{\alpha_k \norm{f'(w_k)}^2} = \frac{1}{2}$, which matches exactly the condition $r_h = 1/2$ in the LossGrad paper.  
The LossGrad algorithm then compute a new step-size $t$ by an increase or decrease from the previous $\alpha_k$ by a factor that requires tuning.

\subsection{Adam}\label{sec:adam}
We start with gradient descent with moment term, which dates back to early days of neural networks   \citep{Rumelhart:1986we}, and has been known to be able to skip local minima. A classical way of wring gradient descent with momentum is
\begin{equation}\label{eq:momentum_classical}
w_{k+1}  = w_k - (\alpha f'(w_k) - p \Delta_{k-1}),  
\end{equation}
where $\Delta_{k-1}= w_k - w_{k-1}$ is the last weight change.
This is exactly equivalent to the following update form: 
\begin{equation}\label{eq:momentum_average}
\Delta_{k} = p \Delta_{k-1} - \alpha f'(w_k), \quad
w_{k+1}  = w_k + \Delta_{k}. 
\end{equation}
This doesn't look very obvious, but it's actually quite simple:
\begin{align*}
w_{k+1} & = w_k + \Delta_{k} \\
& =w_k +  p \Delta_{k-1} - \alpha f'(w_k)\\
& =w_k -  (\alpha f'(w_k) - p \Delta_{k-1})
\end{align*}
The importance of writing the classical form of momentum in terms of equation \ref{eq:momentum_average} is that it is an averaging form that shows the weight change is an exponential smoothing of the (negative) past gradients:
\begin{align*}
\Delta_{k} & =  p \Delta_{k-1} - \alpha f'(w_k) \\
& = -\alpha \sum_{s=0}^k p^{k-s}  f'(w_s).
\end{align*}
When $s=k$, this sum gets $f'(w_k)$;
when $s=k-1$, this sum gets $p f'(w_{k-1})$;
etc. 
The momentum rate $p$ dampens the gradients according to their recency. 

Gradient descent with momentum is known to speed up convergence of gradient descent by improving eigenvalue properties of the underlying O.D.E \citep{10.1016/S0893-6080(98)00116-6} (check for the discrete case in the paper). In particular, when the momentum rate $p \in [0, 1- \sqrt{\alpha k_i}]$, where $k_i$ is the $i$-th largest eigenvalue of the Hessian matrix (assumed to be positive definite), the $i$-th component of a linearly transformed $w$ converges faster than pure gradient descent (with no momentum rate, $p=0$). 

Now let's proceed to ADAM. 
ADAM first smoothes gradient by an incremental version of the following exponential averaging of past gradient, 
\[
m_k = \frac{1}{1-\beta_1} \sum_{s=0}^k \beta_1^{k-s} f'(w_{s}). 
\]
So by updating the weight proportional to $-m_k$, ADAM implements gradient descent with momentum implicitly (with a normalization factor). 
So if we also normalize $\Delta_k$ by $1-p$ in the gradient descent with momentum, we would have $\beta_1 = p$. 
To account for variance,  ADAM also smoothes the squared gradient component-wise:
\[
v_k(i) = \frac{1}{1-\beta_2} \sum_{s=0}^k \beta_2^{k-s} (f'_i(w_{s}))^2. 
\]
Now this is ADAM:
\[
w_{k+1}(i) = w_k(i) -\alpha \frac{m_k(i)}{\sqrt{v_k(i)}}. 
\]  
This normalization idea by $\sqrt{v_k(i)}$ is discussed in an algorithm called RMSprop \footnote{RMSprop was discussed in a lecture note by Geoffrey Hinton, Nitish Srivastava and Kevin Swersky, \url{http://www.cs.toronto.edu/~hinton/coursera/lecture6/lec6.pdf}}, (which is an improvement over normalized gradient method).
Thus ``ADAM is Momentum + RMSprop''. 
ADAM is efficient in computation: everything is performed component-wise and has a linear complexity. 
The term, $\alpha/\sqrt{v_k(i)}$, can be viewed as the component step-size for $w(i)$. 
There are other variant of Adam, such as AdamW, which applies $L_2$ regularization \citep{adamw}. 

\subsection{IDBD}\label{sec:idbd}
Like Adam, IDBD is also a momentum method \citep{sutton1992adapting}. IDBD is an incremental version of Delta-Bar-Delta developed by \citet{jacobs1988increased} with the assistance of Rich Sutton. The principle of both algorithms is that when the current gradient and an average of the previous gradient of a parameter have the same sign, the step-size of that parameter should be increased. 
IDBD adapts step-size incrementally for an online linear regression setting. 
It was presented for least-mean-squares (LMS) in which the loss function is a quadratic mean square loss. 
Define this loss as $E[\delta_k^2(w, \alpha)]$ with $\delta_k = y_k^* - y_k(w, \alpha)$, which is the sampled loss at time step $k$. Here $y_k^*$ is the ground truth for input $x_k$. 
IDBD observes a new sample and loss each time step, and incrementally minimizes this loss.  
It is a parameter-wise update meaning that each parameter has an individual step-size. 
The step-size employs an exponential form to ensure positiveness:
\begin{align*}
\beta_{k+1}(i) & = \beta_k(i) + \eta \left[\delta_k x_k(i)\right] h_k(i) \\
\alpha_{k+1}(i) &= e^{\beta_{k+1}(i)}. 
\end{align*}
The $\eta$ is the only hyper-parameter (the rate of step-size, or meta-step-size).
The term in the bracket is the current gradient at this time step. 
The $h$, called the {\em trace} of recent weight changes, maintains an estimate of the update direction, detailed below. 
Thus the step-size uses the correlation in the current gradient and tracked gradient estimate to guide the adjustment of the step-size.  
After the step-size is updated, the weight $w$ is updated using the new step-size and gradient descent:
\begin{align*}
w_{k+1}(i) &= w_k(i) + \alpha_{k+1}(i) \delta_k x_k(i) 
\end{align*}
Finally, the trace is updated by
\[
h_{k+1}(i) = h_k(i) relu\left(1-\alpha_{k+1}(i) x_k^2(i)\right) + \alpha_{k+1}(i)    \delta_k  x_k(i).
\]
The trace tracks the consistency in the update direction. 
The employment of relu is interesting. 
It keeps considering the past value of $h$ but it occasionally resets this memory. 
The resetting happens when 
\[
\alpha_{k+1}(i) x_k^2(i) \ge 1, 
\]
{\em i.e.}, when relu turns the signal off, which turns off the past. 
This condition prevents the step-size from growing too big (overly optimistic in the consistency of the update).
That why $x_k^2(i)$ was used is kind of mysterious.  
One way of thinking of it is that the step-size should not be bigger than $1/x_k^2(i)$, which will be illustrated shortly.

The trace $h$ is a form of momentum, with a special recency decay. So the key idea of IDBD is to use momentum to adjust step-size. If we are currently in the momentum direction, we increase the step-size; otherwise, we decrease it. 
Assuming that the step size $\alpha(i)$ only affects $w(i)$ (not the other weights), \cite{sutton1992adapting} proved that the update of $\beta$ is gradient descent too, minimizing the same objective as $w$, which is $E[\delta_k^2]$.  

Below we investigate what the assumption that ``$\alpha(i)$ only affects $w(i)$'' means and why the step-size should not be bigger than $1/x_k^2(i)$.  
In fact, this assumption means $\alpha$ is a step-size that updates $w$ in a generic matrix-vector form as
\begin{equation}\label{eq:stepsize-diag}
w_{k+1} = w_k - D_k f'(w_k),  
\end{equation}
where $D_k$ is the diagonal matrix with $diag(D) = \alpha$. 
If there is some step-size $\alpha$ such that each $\alpha(i)$ only affects $w(i)$, then the update of $w$ can be written this way. 
On the other hand, if a step-size leads to an update like this, then each step-size $\alpha(i)$ only affects the $i$-th component.  
In the case of LMS, we have 
\begin{align*}
w_{k+1} & = w_k + D_k \delta_k x_k \\ 
\Longleftrightarrow w_{k+1}(i) & = w_k + \alpha_k(i) \delta_k x_k(i) \\ 
&= w_k(i) + \alpha_k(i) \left( y_k^* - \sum_j w(j) x_k(j)\right) x_k(i) \\ 
&= w_k(i)  - \alpha_k(i)  \sum_j w(j) x_k(j) x_k(i) + \alpha_k(i) x_k(i)y_k^*\\ 
&= \left( 1-  \alpha_k(i) x_k^2(i)\right)w_k(i)  +  \left( \alpha_k(i)x_k(i) y_k^* -  \alpha_k(i) \sum_{j\neq i} w(j) x_k(j) x_k(i) \right). 
\end{align*}
Thus it is clear that a sufficient and necessary condition for every $w_k(i)$ to converge is that $\alpha_k(i) \le \frac{1}{x_k^2(i)}$. The bias term in the bracket will also be finite if this condition holds for all $i$. 
The above analysis applies to any step-size whose update can be written in diagonal form in equation \ref{eq:stepsize-diag}, which includes IDBD. 
Thus it is important for the step-size of IDBD to have $\alpha_k(i) \le \frac{1}{x_k^2(i)}$. 
If this condition does not hold (meaning the step-size of IDBD grows beyond this value at some time step), 
the recency decay, $1- \alpha_k(i) x_k^2(i)$, used for the trace, being negative could cause the algorithm to diverge. 
This explains why IDBD employs relu, to make sure that the decay in $h$ is non-negative. 

Note IDBD still has a meta-step-size parameter $\eta$. An effort to further remove this hyper-parameter was made by Autostep \citep{autostep}. The first modification of IDBD by Autostep is to normalize the update of $\beta$ by a running average of its historical updates. 
The second modification is to normalize $\alpha_k(i)$ by the sum of all the step-sizes at this iteration, weighted by $x_k^2(i)$, so that the error at the $k$th step is reduced without overshoot. 
In reinforcement leanring,
\citet{idbd_td} extended IDBD to TD methods for stationary and non-stationary prediction tasks, which showed better results than TD methods with the constant step-size and scalar step-size adaptation. Metatrace \citep{meta_trace} adapted the step-size of the actor-critic algorithm with eligibility trace for both cases of the scalar step-size and the component step-sizes, which can accelerate the learning process of the actor-critic algorithm. 
The literature called using a step-size for each individual parameter the {\em vector or vectorized step-size}, e.g., see \citep{autostep,idbd_td,meta_trace}. For the reason that the update is equivalent to using a diagonal matrix, we call this technique the {\em diagonal-matrix step-size} or {\em diagonal step-size} for short, which emphasizes the form of the diagonal matrix where the vectorized step-size corresponds to the diagonal part. The diagonal matrix is an important form of step-size in this paper. Later in this paper we are going to see that formulating with the diagonal step-size enables us a better understanding of IDBD and individual component step-sizes.

\subsection{Hypergradient Descent}\label{sec:hd}
Let's look at another step-size adaptation method called hypergradient descent
\citep{baydin2017online}. This method is a functional approach.
The method adapts a scalar step-size though the authors commented that adapting individual step-sizes is also possible. With the gradient descent update:
\begin{align*}
w_{k+1} & = w_k - \alpha f'(w_k(\alpha)) \\
& = w_k - \alpha f'(w_{k-1} - \alpha f'(w_{k-1}(\alpha))) \\
& = w_k - \alpha f'(w_{k-1} - \alpha f'(w_{k-1})), 
\end{align*}
we consider the gradient of the current iteration is a function of the step-size from last iteration (the second line), but only the one-step influence (the third line). Thus the method can be derived by minimizing $E(\alpha) = f(w_{k+1}(\alpha))$. Note the connection to Polyak step-size, which minimizes $\| w_{k}(\alpha) - w^* \|^2$. So by taking the derivative of $f(w_k(\alpha))$ with respect to $\alpha$, we can adapt $\alpha$ according to the gradient descent minimizing $E(\alpha)$:
\[
\alpha_{k+1} = \alpha_k + \eta f'(w_k)^Tf'(w_{k-1}), 
\]
which is the hypergradient descent (HD) algorithm using $\eta$ as the meta step-size parameter. The algorithm means if the gradients of two successive iterations are in an acute angle, then we increase the step-size. This makes sense if one imagines the loss contour is convex and the algorithm's basic hyperthesis is that, increase the step-size as long as we're on the same side of the minimum as the last iteration. 
This is shown in Figure \ref{fig:hd_illustration} (Left).
Finally the step-size increases to a stage that overshoots to the other side of the minimum, and the step-size decreases. So HD uses the correlation of gradients instead of the momentum to guide the adjustment of the step-size. Because the correlation of the gradients is similar to the gradient direction change, HD may be viewed as a second-order method that is efficient to compute. \citet{baydin2017online} also extended the method to a generic update which has a step-size parameter and applied to other gradient descent algorithms.
\begin{figure}[t]
    \centering
    \begin{subfigure}[c]{0.2\textwidth}
    \includegraphics[width=\textwidth]{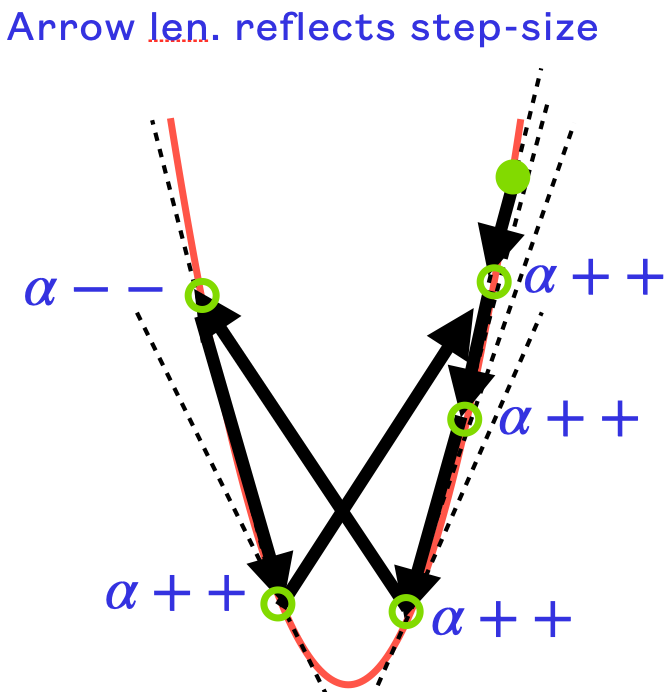}
    \end{subfigure}
    \hfill
    \rulesep
\begin{subfigure}[c]{0.77\textwidth}
    \includegraphics[width=\textwidth]{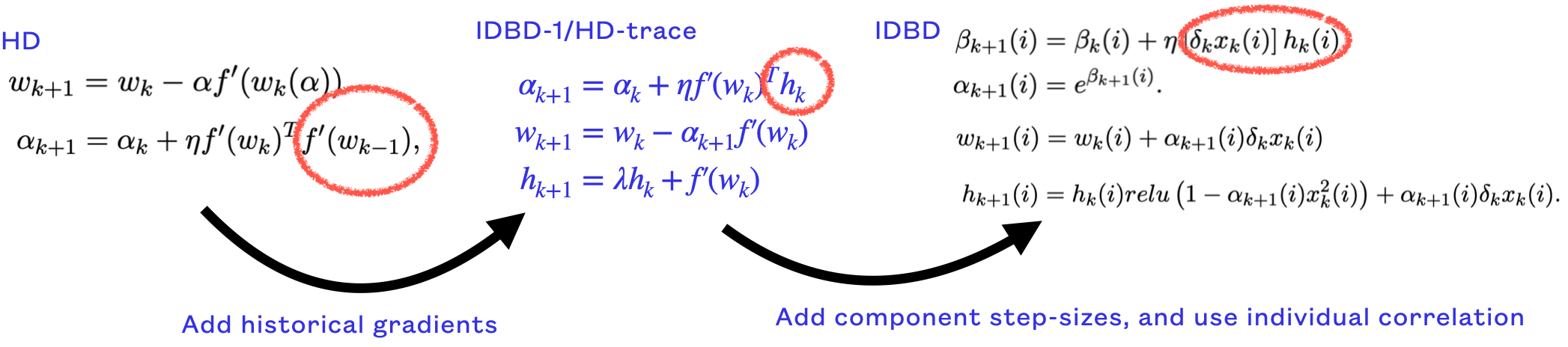}
\end{subfigure}
    \caption{Left: Illustration of step-size increase/decrease with Hyper-gradient (HD) descent where the iteration starts from the green dot. Right:  Understanding the relation between HD and IDBD using IDBD-1 as a bridge. IDBD-1 is a (new) algorithm that uses the smoothed historical gradient to adapt a single step-size by measuring the correlation with the current gradient, where $\lambda \in [0,1)$. When $\lambda=0$, IDBD-1 reduces to HD.}
    \label{fig:hd_illustration}
\end{figure}

Interestingly, as shown in Figure \ref{fig:hd_illustration} (Right), if we compare HD with the updates of IDBD in the linear case, especially the adaptation of $\beta$, we can see that except the difference in that IDBD adapts a diagonal and exponential form of step-size while HD adapts a single step-size, the two algorithms are mainly different in that IDBD uses a recent history of gradient correlation while HD uses the latest two successive gradients to calculate correlation. Thus in this sense, HD is an extension of IDBD beyond the linear case. Because of the use of the diagonal step-size, gradient correlation in IDBD only considers component-wise correlation, and the cross correlation between different components of the gradients is zero. Using the recent gradient correlations instead of only the latest one should be more robust for step-size adaptation, especially because in the setting of stochastic gradient descent measuring whether we are on the same side of the minimum can be noisy. Thus it may be worthwhile to explore the performance of IDBD-1 which is the trace version of HD for deep learning. 

The motivation of the meta-gradient step-size adaptation studies was in a large part on the less or free tuning effort of the step-size.  
However, it is unknown how fast the meta-gradient step-size adaptation methods are in terms of convergence rate. Here we develop an understanding of this question using the case of quadratic optimization. 

\begin{theorem}\label{thm:meta-grad-linear-rate}
For the quadratic optimization problem, with $f(x) = \frac{1}{2}w^T Q w $, where $Q$ is a symmetric positive-definite matrix, with the smallest eigenvalue being $\mu$ and the biggest eigenvalue $L$. Using the optimal step-size found by the meta-gradient step-size adaptation methods leads to a linear rate for $f(w_k) -f(w^*)$, given by $1-\frac{4\mu L}{(\mu+L)^2}$.   
\end{theorem}
\begin{proof}
 Recall that meta-gradient methods aim to find an optimal step-size $\alpha$ by minimizing $f(w_{k+1}(\alpha))$. So first let us express $f(w_{k+1}(\alpha))$ using the gradient descent update and $\alpha$. According to  $f'(w) = Qw$, the definition of $f$, and the symmetry of $Q$, we have 
 \begin{align}
     f(w_{k+1}(\alpha)) &= f(w_k-\alpha f'(w_k) ) \nonumber\\
     &=f(w_k - \alpha Qw_k)\nonumber\\
     &=f((I-\alpha Q)w_k) \nonumber\\
     &=\frac{1}{2} w_k^T (I-\alpha Q) Q(I-\alpha Q) w_k \nonumber\\   
     &= \frac{1}{2} w_k^T (Q-\alpha Q^2) (I-\alpha Q) w_k \nonumber \\
     &= \frac{1}{2} w_k^T (Q-\alpha Q^2 -\alpha Q^2 +\alpha^2 Q^3)  w_k. \nonumber
 \end{align}
 The optimal step-size that meta-gradient methods seek is the one that minimizes $f(w_{k+1}(\alpha))$ with respect to $\alpha$. Let this optimal step-size be $\alpha^*$. By solving $0=\frac{\partial f(w_{k+1}(\alpha))}{\partial \alpha}$, we have 
 \[
 \alpha^* = \frac{w_k^TQ^2w_k}{w_k^TQ^3w_k}.
 \]
 Note this step-size resembles the step-sizes for the minimal residual approach \citep{saad03:IMS} and the preconditioned iterates \citep{ptd_yao}. The common feature of this kind of step-sizes is that there is one more matrix in the denominator than the numerator, serving to have a normalization effect for the gradient.  
 Plug this back to $f(w_{k+1}(\alpha))$, we have
 \begin{align}
 f(w_{k+1}(\alpha^*))
 &= \frac{1}{2}w_k^T(Q - \alpha^* Q^2 - \alpha^*(Q^2 - \alpha^* Q^3))w_k \nonumber \\
 &= \frac{1}{2}w_k^T(Q - \alpha^* Q^2)w_k \nonumber
 \\
 &= f(w_k) - \frac{1}{2}w_k^T\alpha^* Q^2 w_k \nonumber\\
 &= f(w_k) - \frac{1}{2} \alpha^* \norm{f'(w_k)}^2. \nonumber
 \end{align}
Now it suffices to understand how much this error reduction is:
\[
\frac{1}{2} \alpha^* \norm{f'(w_k)}^2 =  \left(f(w_k) - f(w^*)\right) - \left(f(w_{k+1}(\alpha^*)) -f(w^*) \right), 
\]
from which the reduction ratio is
\begin{align*}
\rho_k = 
\frac{f(w_{k+1}(\alpha^*)) -f(w^*)}{f(w_k) - f(w^*)} &= 1- \frac{1}{2} \frac{\alpha^* \norm{f'(w_k)}^2}{f(w_k) - f(w^*)} \\
&= 1- \frac{1}{2} \frac{ \norm{f'(w_k)}^2}{f(w_k)} \alpha^*, 
\end{align*}
because $f(w^*)=0$. We now bring back the matrix $Q$ into the above and obtain
\begin{align*}
\rho_k &= 1-  \frac{w_k^TQ^2w_k}{w_k^T Q w_k} \cdot \frac{w_k^TQ^2w_k}{w_k^TQ^3w_k}.
\end{align*}
We are going to exploit the fact the numerator and denominator both have the powers of $Q$ summing to the same number in a special way. Let $Qw_k=v_k$, then
\begin{align*}
    \rho_k &= 1-  \frac{\norm{v_k}^4}{w_k^T Q w_k} \cdot \frac{1}{w_k^TQ^3w_k}  \\
    &= 1-  \frac{\norm{v_k}^4}{v_k^TQ^{-1} v_k} \cdot \frac{1}{v_k^TQv_k} \\
    & \le 1 - \frac{4\mu L}{(\mu+L)^2}. 
\end{align*}
where third line is according to Kantorovich inequality \citep{householder_theory,horn1985cr}.
This completes the proof. 
\end{proof}
The above proof can be extended to a two-time-scale analysis of meta-gradient step-size adaption methods.
Here we briefly sketch the proof and discuss what it means. Two-time-scale  meta-gradient step-size adaptation methods adapt the step-size faster than $w$. At the steady-state, the step-size converges to $\alpha^*$, and the rate of convergence of $w$ is given exactly in the above theorem. For a rigorous argument, similar proof techniques can be adapted from the contexts of LMS \citep{LMS2} and off-policy actor-critic learning \citep{two_time_scale_shangtong}, which are both based on general two-time-scale stochastic approximation \citep{Konda_2004}. Note this analysis is independent of a specific meta-gradient step-size adaptation method as long as the step-size converges to $\alpha^*$. 
This two-time-scale analysis tells us two things about the meta-gradient approach of adapting a scalar step-size.
\begin{tcolorbox}
(A1.) The meta-gradient approach of scalar step-size adaptation cannot converge faster than the linear rate for quadratic optimization as given in the above theorem. 

(A2.) This meta-gradient approach suffers from the ill conditioning of the problem too, just like the gradient descent with a  constant step-size. 
\end{tcolorbox}
Regarding (A1.), although this rate is slightly faster than gradient descent with the optimal constant step-size, in fact, this linear rate is not even reachable by this meta-gradient step-size adaptation because the gradient descent procedure updating $\alpha$ will spend extra time to find the optimal step-size $\alpha^*$. Thus for L-smooth and strongly convex functions, meta-gradient step-size adaptation methods for adapting a scalar step-size converge slower than this rate too. 

Regarding (A2.), consider $\mu\to 0$, then the rate gets close to $1$, which induces very little error reduction per iteration. One may use the ``preconditioning step-size'' like \citep{ptd_yao} to solve the ill conditioning by letting the gradient descent update of the step-size converge to it. These two findings are the first understanding of the convergence rate of meta-gradient step-size adaptation methods that we know.

\subsection{Meta-gradient}
Note IDBD is also a functional approach, besides being a momentum approach for step-size adaptation. IDBD minimizes the mean-squared-error for LMS for its step-sizes by gradient descent \citep{sutton1992adapting}, just like the normal parameters. 
This approach, using gradient descent to minimize a loss function of the step-size function and simultaneously using another gradient descent procedure to minimize the normal loss function of the parameters, appears to be a popular approach for step-size adaptation, e.g., see our reviews in Section \ref{sec:idbd} and Section \ref{sec:hd}. Beyond step-size adaptation, this idea can be used to adapt other many hyper-parameters, e.g., see \citep{bengio2000gradient} for the case that the hyper-parameter loss function is quadratic and the linear regression problems for stock price prediction.  \cite{hyperparam_gd} further used gradient descent to adapt more training aspects of neural networks, including step-size and momentum schedules, weight initialization distributions, regularization schemes, and even neural network architectures. \citet{bengio2000gradient} also discussed that optimizing the many hyper-parameters may induce more over-fitting, for which there is little progress of answering this question.

Beyond hyper-parameter adaptation, this basic working principle of step-size adaptation such as represented by IDBD has a close relationship to a broad research topic that is normally referred to as ``learning-to-learn'' \citep{thrun1998learning} or ``meta-learning'' \citep{vilalta2002perspective}, which means learning is across tasks whilst each task has some form of base-level learning. 
Meta-learning is focused on generalization across tasks, in particular, by extracting {\em meta-data} from prior tasks and transferring the knowledge to new tasks. Thus it is important that the tasks have some sort of similarity.
The key difference of meta-learning from multi-task learning and ensemble learning is that the latter two tasks do not normally use experience from prior tasks \citep{vanschoren2018meta}. Meta-learning also emphasizes the ability to generalize from task to task with very few samples \citep{thrun1998learning}.  

Recent progress in meta-learning features in a spectrum of gradient descent approach to meta-learning \citep{santoro2016meta,finn2017model,rusu2018meta,rajeswaran2019meta,yin2019meta,lee2019meta,khodak2019adaptive,javed2019meta}.
Traditionally, there are many other approaches to meta-learning besides gradient descent methods. Presumably, the success of deep learning establishes the principle of {\em differentiable and evaluative function approach} to artificial intelligence, by which we mean as long as a complex learning goal can be specified by a differentiable function that can be evaluated on data, then it can be learned by a deep neural networks trained with gradient descent. This means the gradient descent approach to meta-learning is particularly important, backed up by the success of deep learning. It is supposedly much easier to train, test and maintain such systems for practitioners than the rule-based meta-learning methods such as in \citep{chan1993experiments}. 

The gradient descent approach for meta-learning is normally called {\em meta-gradient}, though meta-gradient can be used for other tasks too.  
Every meta-gradient method starts with a set of algorithmic parameters,  task-specific parameters, or task generalization parameters, which are defined according to the problem context. The {\em algorithmic parameters} refer to those that are used by algorithms such as the step-size for gradient descent, and the $\lambda$ factor in the return for a reinforcement learning agent. {\em Task-specific parameters} are those that are particularly important to perform well in an individual task. {\em Task generalization parameters} are the parameters that aim to generalize across tasks, standing for the commonality in tasks. 
Recall that for the case of step-size adaptation methods as we see in our reviews, there are two losses, one loss is the normal one that is induced from the learning parameters, and the other is from the step-size, the algorithmic parameter(s).
Note both losses stem from the same loss function, which is $f$, the loss function of the system. We've seen in the gradient-based step-size adaptation methods such as IDBD and HD that the key is to model how the loss changes with respect to the step-size. The learning parameters are used as a bridge for expressing the loss with respect to the step-size via the gradient descent update. 
This idea is not limited to the step-size and it is surely generalizable, which means for these three classes of parameters, as long as we can express the loss with respect to them, we can perform gradient descent to learn them. Meta-gradient methods can be applied to a combination of these parameters. For example,
\citet{meta-gradient-rl} used gradient descent to adapt the discount factor of the MDP problem and the $\lambda$ factor of the return which is the eligibility trace factor of TD methods \citep{sutton2018reinforcement}. MAML used gradient descent procedures to adapt the task-specific parameters and task generalization parameters \citep{finn2017model}. 
To make the discussions easier, in the following, we call the loss function for these three classes of parameters that are not the normal learning parameters the {\em meta loss}.

Here we briefly review MAML \citep{finn2017model}, especially from the way how the problem is parameterized by the meta-parameters and how the normal loss and the meta-loss functions are formulated.  
MAML has both task-specific  parameters and task generalization parameters. The goal is to generalize across tasks with few samples from each individual task since MAML is a meta-learning algorithm. To achieve this, each task $i$ has an individual parameter vector $\theta^{(i)}$, and there is also a task generalization parameter vector $\theta$ that is supposedly shared by all tasks in order to extract the commonality. At each iteration, MAML visits the set of tasks according to a task distribution. For a task $i$, it samples a number of training data, and uses them to evaluate the gradient of the loss and then update $\theta^{(i)}$ by gradient descent
\[
\theta^{(i)} = \theta^{(i)} - \alpha f'(\theta^{(i)}|X_i, Y_i).
\]
Importantly, $\theta^{(i)}$ is initialized to $\theta$, the task generalization parameter vector; Here $(X_i, Y_i)$ are the samples for task $i$, and $\alpha$ is the step-size. If there is only one task, this just reduces to the normal gradient descent for supervised learning and $\theta$ is just the set of the normal learning parameters. In the general case with more than one tasks, the task-specific loss of MAML is, $f(\theta^{(i)} | X_i, Y_i)$, i.e., the loss evaluated on the task-specific samples and adapted with a small number of steps from the task generalization parameters.  
{\em The idea of MAML is supported by the widely adopted practice of model fine tuning}. Deep neural networks are expensive to train due to the large training datasets and the slowness of gradient descent. So once a deep neural model is trained, it becomes an asset that is usually shared on the web, and then downloaded by someone and fine tuned for an individual application using data therein to perform a usually small number of steps of gradient descent. This way, often effective, enables deep learning practitioners quickly building a decent classifier for their new applications without much training effort and resource. The second loss in MAML is the meta-loss, which is the sum of the individual loss from each task, using the task-specific parameter vector and samples, $f(\theta) = \sum_{i} f(\theta^{(i)} | X_i, Y_i)$, which can be weighted by the distribution over tasks. 
So the update of $\theta$ is simply a gradient descent procedure minimizing $f(\theta)$. 

In a short summary, meta-gradient can be used for step-size adaptation as discussed in Section \ref{sec:idbd} and Section \ref{sec:hd}, meta-learning as discussed in this section, and reinforcement learning for which readers can refer to \citep{meta-gradient-rl,meta-gradient-questions,meta-gradient-sutton}. For more comprehensive coverage on meta-learning, see \citep{vanschoren2018meta,nichol2018first,finn2019online,hospedales2020meta}.

\section{Step-size Planning}\label{sec:step-size-planing}
In this section, we introduce our first step-size planning algorithm Csawg. 
{\em Csawg} stands for {\textbf c}omponent {\textbf s}tep-size {\textbf a}daptation based on {\textbf w}eight difference and squared {\textbf g}radient. 
It is shown in Algorithm \ref{alg:Csawg}.

\begin{algorithm}[t]
\begin{algorithmic}
\State /* This procedure computes a diagonal step-size to speed up SGD 
\State $\gamma$: a scalar step-size in SGD 
\State $B_1, B_2$: buffers for holding recent weight update data that are $K$ steps away */
\State 
\For{$k=1, \ldots, T$}
    \State $w_k \gets w_{k-1} - \gamma g_k$  \quad \quad /* SGD */
    \If{$k \le  K$}
        \State Put $(w_k, g_k)$ into $B_1$ 
    \Else 
        \State Put $(w_k, g_k)$ into $B_2$ 
    \EndIf
    
    \If{$len(B_2) == K$} \quad \quad 
    \State /*Note $B_1=\{(w_1, g_1), \ldots, (w_{K}, g_{K})\}; B_2=\{(w_{K+1}, g_{K+1}), \ldots, (w_{2K}, g_{2K})\}$ */
        \State $sum_{1}, sum_{2} = 0$
        \For{$s =1, \ldots, len(B_1)$}
            \State $sum_{1} \mathrel{{+}{=}} g_s \odot (w_s - w_{s+K} )$\quad \quad /*  element-wise product */
            \State $sum_{2} \mathrel{{+}{=}} g_s \odot g_s$ 
        \EndFor
        \For{$i =1, \ldots, len(w)$}
        \If{$sum_2(i)==0$}
            \State Set $\alpha(i) =0$ 
        \Else
            \State Compute $\alpha(i) = \frac{sum_{1}(i)}{sum_{2}(i)}$
        \EndIf
        \EndFor
        
    \For{$i=1, \ldots, len(w)$}
        \State $w(i) \gets w(i) - \alpha(i) g_k$ 
    \EndFor
    
    \State $B1 \gets B_2$
    \State Emptify $B_2$ 
        
    \EndIf

\EndFor

\end{algorithmic}
\caption{Csawg Algorithm: {\bfseries C}omponent {\bfseries s}tep-size {\bfseries a}daptation based on {\bfseries w}eight difference and squared {\bfseries g}radient.}
\label{alg:Csawg}
\end{algorithm}

Csawg is a novel way of speeding up convergence for gradient descent.  
The goal is to optimize a function $F: \real^d \to \real$. Assume there is a random function $f: \real^d \to \real$ which we have access to, and $Ef(x)=F(x)$, for all $x \in \real^d$. 
This covers typical settings like deep learning and deep reinforcement learning, in which the true loss function is not given but revealed step by step on mini-batch samples. 
To introduce the algorithm, we start with the error of the weight.
Assume $g_k=f'(w_k;x_k)$ is the stochastic gradient of the loss for a sample $x_k$. 
We consider the {\em diagonal SGD} update:
\begin{equation}\label{eq:diagstep}
w_{k+1} = w_{k}- \alpha_k g_k, 
\end{equation}
where $\alpha_k$ is a diagonal matrix with each diagonal entry being the step-size for the corresponding component of $w$. The diagonal step-size form was used by Adam and IDBD, see Section \ref{sec:adam} and \ref{sec:idbd}.

\begin{theorem}
Let $w_k$ be generated by the procedure in equation \ref{eq:diagstep} which uses a diagonal matrix as the step-size.
Assume the convergence point of $w_k$ is $w^*$. 
Then the ideal step-size that minimizes $E \norm{w_{k+1}(\alpha) - w^*}^2 $ is given by
\begin{equation}\label{wstar_step_size}
\alpha_k(i,i) = \frac{E g_k(i) \Bigl( w_k(i) - w^*(i)\Bigr) }{E g_k^2(i)}, \quad i=1, 2, \ldots, d.
\end{equation}
\end{theorem}
\begin{proof}
Plugging in equation \ref{eq:diagstep},
the expected error of $w$ after the update is 
\begin{align}
E \norm{w_{k+1}(\alpha) - w^*}^2 &=   E (w_{k} - w^* - \alpha g_{k})^T (w_{k} - w^* - \alpha g_k) \nonumber\\
&= E \norm{w_{k} - w^*}^2 - \Bigl(2 E (w_{k} - w^*)^T \alpha g_k  - Eg_k^T\alpha^2 g_k\Bigr). \label{eq:errort+1}
\end{align}
Consider just an individual component. 
Note this is a quadratic function of $\alpha(i, i)$, whose derivative is
\[
-\frac{\partial E \norm{w_{k+1} - w^*}^2}{\partial \alpha(i, i)} =
2 E\left(w_{k}(i) - w^*(i)\right) g_k(i) - 2 Eg_k^2(i)\alpha(i, i). 
\]
Setting this to zero, we obtain the the desired step-size for the $i$th component. 
\end{proof}
Recall that the derivation of Polyak step-size also starts from the error in the weight and it derives a step-size that depends on the loss difference, $f(w_k) -f(w^*)$, by assuming the convexity of $f$, see our review in Section \ref{sec:review}.
Here we don't assume any condition on $f$ except the derivability of $f$, and thus the derived step-size applies to more general classes of loss functions.

\begin{figure}
\centering
\begin{subfigure}[c]{0.49\textwidth}
\centering
    \includegraphics[width=2.4in]{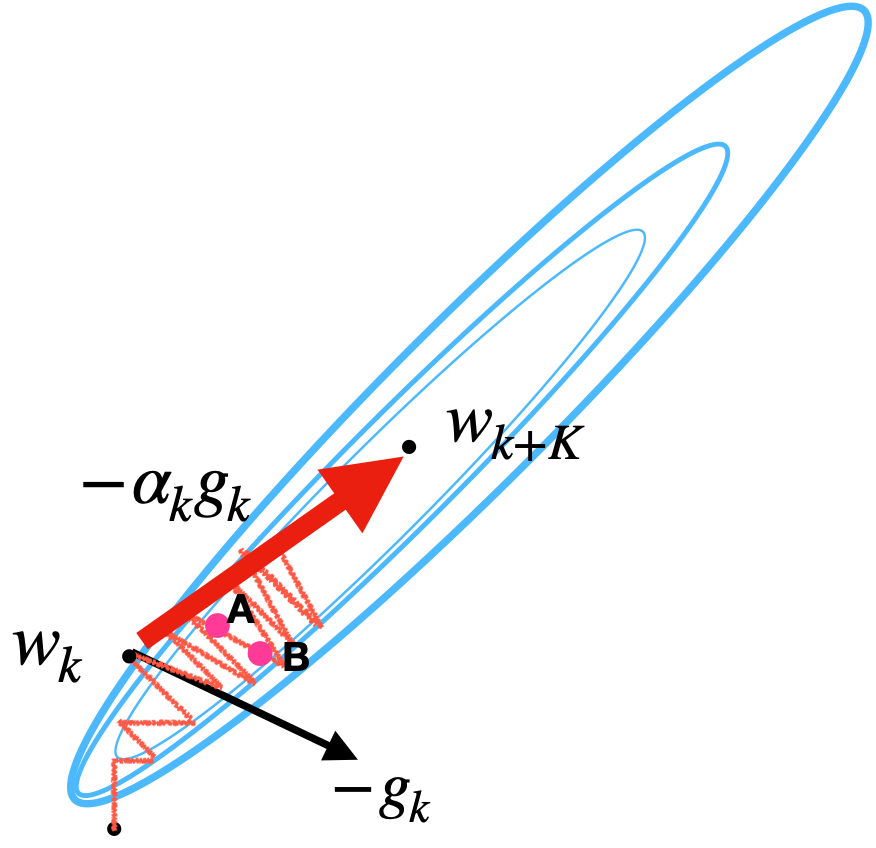}
\end{subfigure}
\hfill
\begin{subfigure}[c]{0.49\textwidth}
\centering
    \includegraphics[width=2.2in]{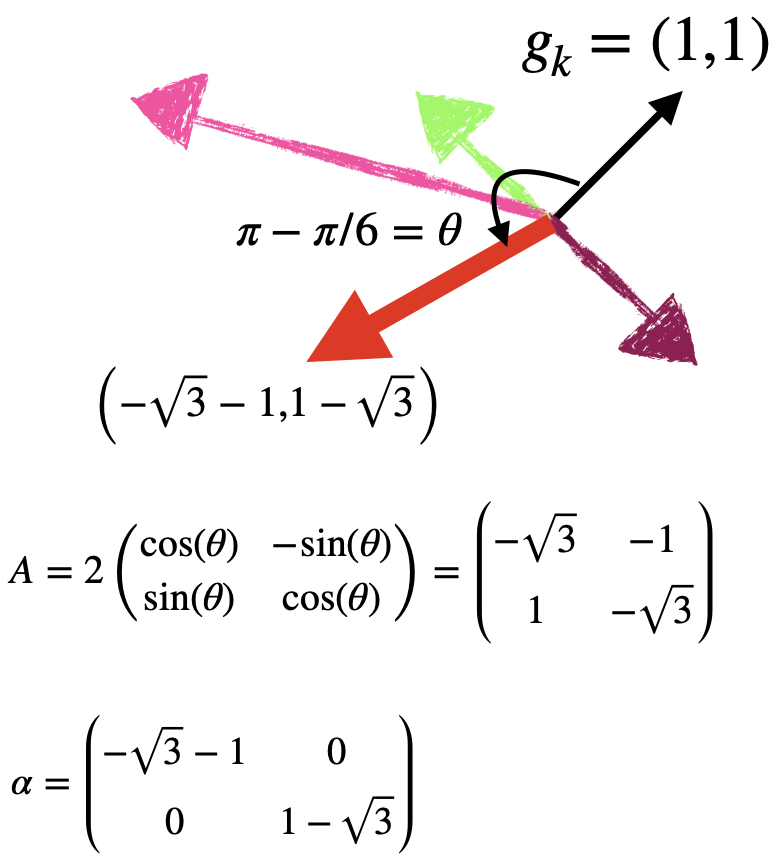}
\end{subfigure}
\caption{Illustration of Csawg. Left: Csawg bootstraps with a diagonal matrix form of step-sizes that ``transitions'' the current weight multiple steps into the future. Right: the diagonal step-size can apply versatile transformations to a gradient vector. See the red bold arrow for the desired vector transformed from the black vector, $g_k$, which can be obtained by projecting a diagonal matrix $\alpha$ instead of a dense matrix $A$.}
\label{fig:Csawg_illustration}
\end{figure}

An interesting observation here is that the step-size has a denominator which is the mean of the {\em squared gradient}, instead of the {\em squared root} of the term as used in Adagrad \citep{adagrad,li2019convergence}, AdaDelta \citep{adadelta}, RMSprop \citep{rmsprop}, ADAM \citep{kingma2017adam}, AMSGrad \citep{amsgrad} and GGT \citep{GGT_svd_lowrank_fullmatrix}. The v-SGD algorithm \citep{schaul2013pesky} has this squared term in the denominator for a quadratic function whose Hessian is diagonal. 
Dividing by the squared gradient makes the step-size more sensitive to the gradient magnitude, increasing the step-size much more for vanishing gradients and decreasing more when the slope is big. Also note that in IDBD, the convergence condition $ \alpha_k(i)\le \frac{1}{x_k^2(i)}$, is the inverse of the squared gradient too.

Note that $E g_k \neq 0$ although $\lim_{k\to\infty} E g_k =0$. Thus unfortunately equation \ref{wstar_step_size} requires the knowledge of $w^*$ in order to compute the step-size. Our idea to circumvent this difficulty is a {\em planning architecture} that updates the parameters occasionally from a multi-step model learned from sampled updates of SGD. To see this, this procedure aims to bootstrap multiple steps ahead of SGD updates:
\[
w_{k+K} \leftarrow w_k  - \alpha_k g_k,
\]
where $w_{k+K}$ is the weight at iteration $k+K$, and the right-side is the weight update by using the diagonal form of step-size.  
That is, we want to advance the normal SGD update with a diagonal-matrix step-size projection, hopefully resulting in a multiple-times faster convergence for SGD. This is illustrated in the sketch on the left of Figure \ref{fig:Csawg_illustration}.

Let us define a sampled loss for the step-sizes:
\[
g(\alpha, w_k; x_k) = \alpha^T G_k \alpha - 2\alpha^T b_k,  
\]
with $\alpha$ being the vector form of the step-sizes,  $G_k$ is a diagonal with $G_k(i, i) = g_k^2(i)$ and $b_k(i) = g_k(i)(w_k(i) -  w_{k+K}(i))$. These two quantities are defined as sampled approximations for the denominator and numerator of equation \ref{wstar_step_size}, respectively. 

We collect $n$ sampled updates at time steps $k_1, k_2, \ldots, k_n$. The empirical loss function for the step-size of these sampled updates is the sum of $g$ for each sample:
\begin{align*}
    G(\alpha, w_{k_j}) &= \sum_{j=1}^n g(\alpha, w_{k_j}; x_{k_j}) \\
    &= \sum_{j=1}^n  \alpha^T G_{k_j} \alpha - 2\alpha^T b_{k_j} \\
    &= \alpha^T G \alpha - 2 \alpha^T b,
\end{align*}
with $G = \sum_{j} G_{k_j}$, which is also a diagonal matrix, and $b = \sum_j b_{k_j}$. Then the loss can be minimized by $\alpha = G^{-1} b$, which is $O(d)$ to compute; recall that $d$ is the number of parameters. Algorithm \ref{alg:Csawg} shows the case when these sampled updates are collected online, with the samples from the most recent updates. 
The SGD update is performed every iteration, and the acceleration update using the diagonal step-size is performed every $K$ iterations.

{\em Zero gradient component}. Note $G$ may be not invertible due to that some diagonal entry of $G$ may encounter zero. We could introduce a positive perturbation in the denominator in computing the component step-size in Algorithm \ref{alg:Csawg}, but this introduces an extra hyper-parameter. Instead, the denominator being zero means all the gradients in the experience is zero --- indicating that the SGD may do a good job for this component at this time step. To decide projection for this weight component, we need more experience from SGD update. Does this component already reach its stationary point, or is it temporary and future SGD update will further take it away from here? We put these questions on hold, and apply no projection to the weight component at this time step until further information comes in.

{\em Negative Step-size}. 
In the vast literature of gradient descent algorithms, it is widely accepted to use non-negative step-sizes in both theoretical analysis and applications. 
For example, the step-size that is positive and diminishes at a certain rate with the iteration number of a stochastic approximation procedure \citep{robbins1951stochastic} is a standard for analyzing machine learning algorithms \citep{bertsekas2000gradient,tsi_td,borkar2000ode,Konda_2004,duchi_2019stochastic,parameter_free_2022making}. Diminishing and small constant step-sizes are widely used in implementing algorithms such as gradient descent algorithms and reinforcement learning algorithms. The diminishing step-size is important to suppress the noises in the stochastic updates. Small constant step-sizes has a similar effect but the converged solution has a bias. 
Faster algorithms without bias were made possible in different contexts, by building certain data structures which enable the use of big constant step-sizes, e.g., see LSPE \citep{lspe96,lspe03} for the case of TD learning, and SVRG for the case of stochastic gradient descent \citep{svrg}.
Line search methods such as Armijo method find the best step-size in a geometric series of positive step-size and it has a wide applications in machine learning and optimization, e.g., see \citep{armijo1966minimization,luenberger1984linear,kolda2003optimization,changshui_linesearch,vaswani2019painless}.

The step-sizes computed by Csawg can be negative at some moments. We don't enforce non-negativeness of the step-size with Csawg though early works like IDBD \citep{sutton1992adapting}, AutoStep \citep{autostep},  Metatrace\citep{meta_trace}, and Tidbd \citep{idbd_td} all used an exponential form of step-size  to guarantee this and SMD \citep{smd_schraudolph1999local} used a thresholding to ``safeguard unreasonably small and negative values (of step-sizes)''. Interestingly, this doesn't prevent Csawg from reducing errors and guaranteeing convergence. One could certainly have different ways of encoding step-sizes to be positive. However, it appears allowing only positive step-sizes limits learning speed and it's not the best to do for our learning update process. Let's take a step back. 

\begin{tcolorbox}
(B.) {\em Why are we restricted to non-negative step-sizes anyway}? The norm of using non-negative step-sizes is presumably a historical inheritance from classical, deterministic gradient descent which often uses a scalar step-size, e.g., see \citet{curry1944method}'s analysis\footnote{Perhaps the name ``step-size'' indicates that the concept was meant to be non-negative, as gradient descent can be viewed as a special Euler method for solving ordinary differential equations in which there is also a ``step'' parameter that determines how far to travel from the last step. The deep learning communities often call the step-size as ``learning rate'', which also indicates non-negativeness. In IDBD, \citeauthor{sutton1992adapting} used an exponential form of step-sizes and he commented that ``it is a natural way of assuring that $\alpha_i$ will always be positive. ''}, and  their wide applications in stochastic approximation \citep{robbins1951stochastic,yin1997stochastic,hanfu_2006stochastic},
optimization \citep{nesterov2003introductory,boyd_book,bubeck_book},
signal processing \citep{ljung1998system},
neural networks \citep{Rumelhart:1986we,simon_book}, pattern recognition \citep{pr_bishop},
machine learning \citep{hastie2009elements}, neuro-dynamic programming \citep{bertsekas1996neuro}, reinforcement learning \citep{sutton2018reinforcement}, deep learning \citep{lecun2015deep,goodfellow2016deep},
deep reinforcement learning \citep{mnih2015human,silver2016mastering,dist_rl,li2017deep,arulkumaran2017brief,franccois2018introduction};
etc.

\quad\quad In the stochastic setting, gradient descent direction can be very bad. To reduce errors fast, when a stochastic gradient is poor (e.g., not in an acute angle but instead in an obtuse angle to the gradient of the underlying true loss function), a positive step-size for the update will actually do gradient ascent instead! Thus if this happens, negative step-sizes can lead to faster convergence for SGD. 

\quad\quad Even in the deterministic setting, traveling always in the negative gradient direction easily leads to oscillation, especially for ill-conditioned problems, noting that the negative gradient does not necessarily point to where the minimum is located. 
Allowing both negative and positive values in the diagonal step-size enables the update to travel in the positive gradient direction along which oscillation happens and travel in the negative gradient along which the minimum is located. 
\end{tcolorbox}
Check the loss contour in Figure \ref{fig:Csawg_illustration} (Left) for an understanding of why we need negative component step-sizes even in the deterministic setting. Note the oscillation in the update causes our estimation of $w$ to lose progress because of the backward motion in the negative x-axis or negative y-axis from time to time. For example, in the figure, the update from point $A$ to point $B$ makes progress in the x-axis but it also makes a backward motion in the y-axis. Let's say the x-axis is $w(1)$ and y-axis is $w(2)$. 
When this happens, a negative step-size for $w(2)$ and a positive step-size for $w(1)$ helps the update move consistently towards the minimum. This basically means certain components of the parameter vector need to occasionally ``back off'' and not always go along in the negative gradient direction during update. In other words, {\bfseries with both positive and negative component step-sizes, some parameters go gradient descent while the others go gradient ascent for faster learning in the perspective 
of the sampled loss}.  (This appears to be very weird even to myself as I'm typing down this after illustrating it in an example above.) Importantly, the signs of the Csawg  step-sizes are not fixed for the parameter components and they change from time to time.    

In our experiments, we have a deterministic gradient descent example in which some component step-size learned by Csawg being ``momentarily'' negative (not always negative of course) leads to much faster learning than gradient descent. 
To get some intuition of this, let us examine the loss in equation \ref{eq:errort+1} with the step-size defined in equation \ref{wstar_step_size}. One can show that
\begin{align*}
E \norm{w_{k+1} - w^*}^2 &=   E \norm{w_{k} - w^*}^2  - E g_k^T \alpha_k^2 g_k. 
\end{align*}
Thus if certain step-sizes in $\alpha_k$ are negative, a reduction still incurs because the reduction is dependent on the {\em squared} step-sizes. The case of Csawg which  does not use $w^*$ is less obvious but the intuition is similar.

{\em The expressiveness of diagonal transformation and diagonal step-size}. The diagonal form of step-size for SGD essentially uses an individual step-size for each parameter, which was used for neural networks training \citep{jacobs1988increased} and non-stationary tracking problems \citep{sutton1992adapting}. 
It is tempting to think that the diagonal step-size form is limited in that it assumes updates of different components are independent and one has to resort to a full matrix to handle the correlation in the components. For example, the success of Newton methods in various scientific computing domains may give us such an impression. 
However, a full matrix form of step-size is at least $O(d^2)$ in both storage and computation. This is not affordable for deep learning applications, with recent big networks having hundreds of billions parameters. There are many algorithms in classical numerical analysis and recent advances in deep learning that apply an approximation of the inverse of Hessian to the gradient for fast convergence. Hessian and its inverse are large, and most of the time, dense. Using finite difference approximation of the Hessian vector product does not require a Hessian matrix, but evaluating the product requires sweeping the data set, which is expensive. Some low-rank approximation of Hessian is usually used in standard numerical analysis like BFGS method, which are usually referred to as truncated Newton or ``Hessian-free (HF)''. HF methods solve a linear system $(H+ \lambda I)d = g$ (with conjugate gradient), where $g$ is the gradient, and uses $d$ to update the weight instead of $g$. They usually require the $H$ matrix to be a good approximation of Hessian which needs to be semi-positive definite. 

Recent advances of speeding up gradient descent for deep learning in general have two directions. {\em The first category of algorithms approximate the product of the Hessian matrix with a vector} so that the computation can be fast. 
\cite{vinyals2011krylov} used an approximate Hessian that does not require the matrix to be semi-positive definite, and constructs a Krylov basis to approximate the product. 
The R-operator method \citep{Pearlmutter} can compute the product without the Hessian for functions that are twice differentiable. It has been used to compute the spectrum of Hessian matrix in training deep networks \citep{lecun_hessian,yao2020pyhessian}, but we couldn't find any  result of using this operator in accelerating SGD. 
{\em The second category is about making the inverse of the Hessian matrix more efficient, through block diagonalization or low rank approximation}. \citet{becker1988improving} are probably the first to use diagonal-matrix approximation to the Hessian matrix in neural networks to speed up the convergence of gradient descent. 
\cite{mathieu2014fast} represented an approximation of Hessian with an orthogonal matrix for rotation and a diagonal matrix according the factorization of symmetric matrices so that the inverse of the approximate Hessian can be computed without $O(d^2)$ complexity by using linearithmic rotations. 
K-FAC \citep{Kfac} approximates the blocks of a neural network's Fisher information matrix which needs not to be diagonal or low-rank, by using Kronecker product of two much smaller matrices. The natural gradient \citep{pascanu2014revisiting} is then applied using a block-diagonal and block-tridiagonal approximation of the Fisher matrix. The GGT algorithm \citep{GGT_svd_lowrank_fullmatrix} is based on the full-sized matrix from the auto-correlation of the gradient vectors and uses low-rank approximation of its squared root without explicitly forming the matrix. The key idea is to use a small window of historical gradient vectors to induce a low-rank matrix so that Singular Value Decomposition can be performed on a much smaller matrix whose size is determined by the window. This is interesting because it can be viewed as using the update experience in a special way. 

\begin{figure}[t]
    \centering
    \includegraphics[width=\textwidth]{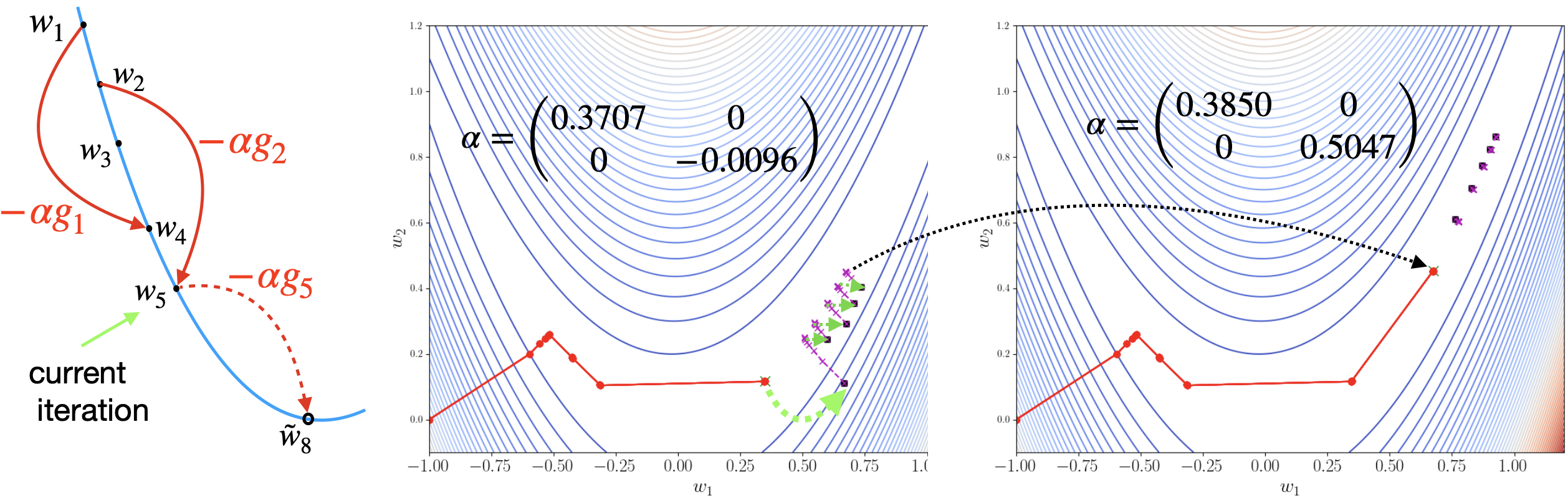}
    \caption{Csawg K3 (with a three-step step-size transition model). Left: illustration of weight differences and diagonal transformation of the gradients to make predictions. Middle and right: diagonal step-size models learned for Rosenbrock function and using them for projection (green arrows): shown are two successive iterations. 
    Note the big step-sizes for component 1, and an negative step-size and a big step-size for component 2 at the two moments, which are both learned from the update data. 
    See the text in Section \ref{sec:multi-step} for detailed explanations.}
    \label{fig:pairing_prj}
\end{figure}

These progresses are by no means a comprehensive coverage but they are examples of promising results of adapting classical second-order methods for deep learning.
Some of them also take advantage of certain diagonal matrices for fast computing of Hessian inverse. However, these diagonal matrices are constructed from the perspective of approximating Hessian, which may limit their transformation power to the gradient vector. 
Diagonal matrix is quite expressive for transformation despite its simplicity, which may be counter intuitive at first sights.
As shown in the right plot of Figure \ref{fig:Csawg_illustration}, to transform a gradient vector (black arrow) to the desired vector (red arrow), from the linear transformation theory one has to apply a rotation and a scaling, which gives a full matrix.  
However, the diagonal matrix can do the same job too, in particular:
\begin{tcolorbox}
(C.) This expressiveness enables the diagonal step-size versatile transformations to ``non-degenerate'' (non-zero) gradient components\footnote{This gives another explanation of why ``every weight in a network should be given its learning rate and that these learning rates should be allowed to adapt over time'' according to \citet{jacobs1988increased} besides being suggested by his four heuristics for step-size adaptation. A single scalar step-size cannot do this job while a full-sized matrix is too heavy for memory and computation purposes.}, changing a non-degenerate gradient vector to any direction with any length, including the orthogonal directions and the opposite direction. 
\end{tcolorbox}
If some gradient component is consistently zero from the experience, it turns out this indicates we may have already reached the stationary point for this weight component, or it is just momentary. 
In either case, we don't need to transform this gradient component at this time of planning. By doing this, we avoid the disastrous division-by-zero error without introducing hyper-parameters, by waiting for more information in the future about this component (e.g., providing a non-zero gradient the next time planning is called for) to apply projection later. See Algorithm \ref{alg:Csawg} and the discussions in the paragraph ``Zero gradient component''.  
With this consideration for degenerate gradient components, the diagonal step-size has the same transformation power of a full matrix for our purpose. Importantly, the storage and the computation of applying it to the gradient is only $O(d)$, which is highly scalable for large scale applications such as deep learning. 

Our intuitive explanation of the projection power of the diagonal matrix and empirical results using it in practice especially deep learning are evidence that the diagonal-matrix step-size seems to be important, but they do not justify confidently to the level that the diagonal-matrix form of the step-size is better than the scalar step-size. 
We question ourselves the following, ``why do we need the diagonal step-size at all? Does it do any goodness that the scalar step-size cannot do?'' There still lacks a mathematical understanding about the role of the diagonal step-size. Below we study these questions by still using the quadratic optimization problem. 

\begin{theorem}\label{thm:meta_grad_diag_alpha}
Consider the quadratic optimization problem as described in Theorem \ref{thm:meta-grad-linear-rate}. Instead of solving the problem with an optimal scalar step-size, let us consider now $\alpha$ is a diagonal step-size. Then we have, at iteration $k$, $\alpha^* = \arg\min_{\alpha \in \mathcal{D}} f(w_{k+1}(\alpha))$, where $\mathcal{D}$ is the space of the diagonal matrices, with
\[
\alpha^*(i, i) = \frac{w_k(i)}{(Q w_k)(i)}, \quad i =1, \ldots, d. 
\]
The convergence rate of using this  $\alpha^*$ is zero, which means it finds the optimal solution in just one iteration. 
\end{theorem}
\begin{proof}
Note $\alpha$ is now a diagonal matrix. 
\begin{align}
     f(w_{k+1}(\alpha)) 
     &=f((I-\alpha Q)w_k) \nonumber\\
     &=\frac{1}{2} w_k^T (I-\alpha Q)^T Q(I-\alpha Q) w_k \nonumber\\ 
      &=\frac{1}{2} w_k^T (Q-Q\alpha Q) (I-\alpha Q) w_k \nonumber\\
     &= \frac{1}{2} w_k^T (Q-Q\alpha Q -Q\alpha Q +Q \alpha Q \alpha Q)  w_k. \nonumber
 \end{align}
 It remains to do two things: 1) solving $\alpha^*$; and 2) showing the reduction ratio.  
 Let $y_k= Qw_k$. 
We will solve $\alpha^*(i, i)$ from minimizing $f(w_{k+1}(\alpha))$, for which have
\begin{align*}
f(w_{k+1}(\alpha)) 
&= f(w_k)- \sum_i \alpha(i,i)  y_k^2(i) + \frac{1}{2}\sum_{i,j} (\alpha Q \alpha)(i,j) y_k(i)y_k(j) \\
&= f(w_k)- \sum_i \alpha(i,i)  y_k^2(i) + \frac{1}{2}\sum_{i,j} \alpha(i,i) Q(i,j) \alpha(j,j) y_k(i)y_k(j)
\end{align*}
Let $\alpha_V=diag(\alpha)$, i.e., the vector of the diagonal entries. Then the solution to $0=\frac{\partial f(w_{k+1}(\alpha))}{\partial \alpha}$ is the one that solves $A \alpha_V = b$,
where 
\[
A(i,j) = Q(i,j) y_k(i) y_k(j), \quad b(i)= y_k^2(i).
\]
Note $A$ is also a symmetric matrix.
Let $D_k = Diag(y_k)$, i.e., the diagonal matrix formed by using $y_k$ as the diagonal part. Then 
\begin{align*}
\alpha_V 
&= A^{-1}b \\
&= (D_k Q D_k )^{-1} b\\
&= D_k^{-1}Q^{-1} D_k^{-1} b\\
& = D_k^{-1}Q^{-1} D_k^{-1} (D_k y_k)\\
& = D_k^{-1}Q^{-1} y_k\\
& = D_k^{-1}Q^{-1} (Qw_k)\\
& = D_k^{-1} w_k,
\end{align*}
where the fourth line is because $b=D_k y_k$. This gives $\alpha_V(i) = \alpha^*(i,i)$ in the theorem.  
Now the new error is:
\begin{align*}
f(w_{k+1}(\alpha^*)) - f(w^*)
&= f(w_k)  - \frac{1}{2} w_k^T  Q  \alpha^* Q w_k \\
&= f(w_k)  -  \frac{1}{2} \sum_i \alpha^*(i,i) 
y_k^2(i) \\
&=f(w_k)  -  \frac{1}{2} \sum_i \frac{w_k(i)}{y_k(i)} 
y_k^2(i) \\
&= f(w_k)  -  \frac{1}{2} \sum_i w_k(i) y_k(i) \\
&= f(w_k)  -  \frac{1}{2} \sum_i w_k(i) (Qw_k)(i) \\
&= f(w_k)  -  \frac{1}{2} w_k^T Qw_k \\
&= 0.
\end{align*}
Thus the error turns to zero in just one iteration given an initial $w_0$. 
This completes the proof. 
\end{proof}
Recall that Theorem \ref{thm:meta-grad-linear-rate} also derives the step-size by minimizing the new loss, $f(w_{k+1}(\alpha))$ with respect to $\alpha$. The only change in Theorem \ref{thm:meta_grad_diag_alpha} is that the step-size is now a diagonal matrix instead of a scalar step-size as in Theorem \ref{thm:meta-grad-linear-rate}. 
As Theorem \ref{thm:meta_grad_diag_alpha} shows, this results in a big difference. 
In particular, using optimal component step-sizes for each parameter solves the problem in just one iteration. This is surprising, especially for the general case where $Q$ is not diagonal. This means for faster convergence, it makes sense for meta-gradient methods like IDBD to have an individual step-size for each parameter; while adapting a scalar step-size like  hyper-gradient descent algorithms are slower.  

To develop a quick understanding of this theorem without looking at the proof in details, consider the special case that $Q=diag(\mu, L)$. Now the error reduction ratio is,
\begin{align*}
    \rho_k 
    &= 1- \frac{w_k^T Q \alpha^* Q w_k}{w_k^TQw_k} \\
    &= 1 - \frac{w_k^2(1)\alpha^*(1,1) \mu^2 + w_k^2(2)\alpha^*(2,2) L^2 }{w_k^TQw_k}.
\end{align*}
Note in this case, $\alpha^*(1,1) = \frac{1}{\mu}$ and $\alpha^*(2,2)=\frac{1}{L}$. This is a simple and beautiful fact in the opinion of the author. It is the key to that the diagonal step-size does not suffer from ill conditioning for this fundamental problem. Therefore, 
\begin{align*}
\rho_k 
    &= 1 - \frac{w_k^2(1) \mu + w_k^2(2) L }{w_k^TQw_k} = 0.
\end{align*}
That is, the diagonal step-size finds the optimal solution in just one iteration for this special but fundamental case, independent of the conditioning number $L/\mu$ of the problem. This is in a great contrast to using an optimal scalar step-size as shown in Theorem \ref{thm:meta-grad-linear-rate} (with a linear rate of about $1-\frac{4\mu}{L}$), and gradient descent using a constant step-size (with a rate $1-\frac{\mu}{L}$, e.g. see Section \ref{sec:conv_rate}), and Nesterov's accelerated gradient, with a linear rate of $1-\frac{\sqrt{\mu}}{\sqrt{L}}$, which is the best rate achieved with a constant scalar step-size. 
In the more general case of $d$-dimensional diagonal matrix $Q$, $\alpha^*$ happens to be equal to just $Q^{-1}$, or the diagonal matrix with the entries being the inverse of the eigenvalues of $Q$. 
Regardless of the many methods (like what we reviewed and discussed above) that already employ the diagonal-matrix form of the step-size, this is the first theoretical understanding of why diagonal step-size is necessary, especially for ill-conditioned problems. In the two-dimensional case, the simple fact that one optimal step-size equal to $\frac{1}{\mu}$ and the other equal to $\frac{1}{L}$ shows we cannot achieve the same optimality (i.e., super fast error reduction) with a single scalar step-size. It also explains why the scalar step-sizes are required to be in the range of $[\frac{1}{L}, \frac{1}{\mu}]$, whether a constant step-size, or Polyak step-size\footnote{One can prove that Polyak step-size is in this range for $L$-smooth and $\mu$-strongly-convex functions.}.
In a summary, the message of this theorem is the following:
\begin{tcolorbox}
(D.) The wide practice of using a scalar step-size for gradient descent is not good for fast convergence because it is {\bfseries over-streched}, in that it tries to balance between the smallest and biggest optimal step-sizes of all the parameters for the convergence guarantee, which, however, is much slower due to that each parameter needs an individual optimal step-size. 
\end{tcolorbox}

This is first formal understanding of \citet{jacobs1988increased}'s idea that each parameter should have their individual step-size. It is a well known and long standing question in optimization that gradient descent suffers from the ill conditioning of the loss function, which leads to extreme slowness. 
As shown by this theorem, the diagonal step-size is a promising solution to this problem. 
Step-size adaptation methods gives us less or free effort in hyper-parameter tuning, e.g., see  \citep{mathews1993stochastic,NLMS,hutter2007temporal,vanschoren2018meta,chen2020better}. 
These works adapt a scalar step-size which gives a close-to or matching performance to the tuned step-sizes or the optimal one.
However, the scope of step-size adaption is bigger than just in the tuning efforts. 
For the diagonal step-size, even faster learning results than the best tuned scalar step-sizes have been reported, e.g., in IDBD \citep{sutton1992adapting}, SMD \citep{smd_schraudolph1999local}, TIDBD \citep{idbd_td} and Metatrace \citep{meta_trace}. Although the host algorithms for step-size adaptation are different in these works, they all reported better learning results in their context than the tuned or even the optimal scalar step-size for the host algorithm.  
These results were surprising in a non-intuitive sense and there has been no explanation towards this phenomenon. This theorem considers a basic but important setting and shows that when we allow them to be adapted to minimize the loss, the diagonal step-size is theoretically much faster than the scalar step-size so much so that the representation of the diagonal step-size enables it to find the solution in just one iteration. Certainly meta-gradient methods like those we discussed won't solve their problems in one iteration  because the context of their problems is stochastic and also the meta-gradient needs to find the optimal step-size for each component which takes time. However, this theorem shows that the effort of {\em the parameter gradient descent} is much smaller when the diagonal step-size is used. The major effort is {\em the step-size gradient descent}, which is the learning of the optimal component step-sizes for each parameter using gradient descent.
Therefore, step-size adaptation via the form of a diagonal-matrix step-size is an important acceleration technique for (parameter) gradient descent.  
A deeper understanding of the diagonal step-size adaptation with a convergence rate analysis into the above algorithms in their context is an interesting future work. 

For non-diagonal matrix $Q$, note that $\alpha^*(i, i)$ can be negative. If there is some method that approximates $\alpha^*(i, i)$, then the convergence won't take just one iteration. If $\alpha^*(i, i)$ is negative, this approximation is also likely to be negative. 
Thus this theorem also validates our previous discussion on using negative step-sizes. 

\section{The Data Perspective of Gradient Descent}\label{sec:relation_rl_planning}
Csawg uses the {\em update experience} to learn an improved way of updating the parameters. It organizes the experience into $K$ steps away from each other to facilitate planning. From the past experience, Csawg learns the step-size which is a form of multi-step machine that predicts future updates. In this sense, the idea of Csawg is similar to planning in artificial intelligence, in particular, Dyna \citep{dyna} and linear Dyna \citep{lin_dyna}. In the Csawg architecture, SGD performs online updates and provides experience. Simultaneously, the update experience is collected and used to learn a {\em world model}, which is the step-size, to reason about the multi-step transition between updates.
Using this step-size transition model, the future update is predicted in the background and passed to the online update to teleport SGD. The multi-step transitions via diagonal transformation by the step-size model including data collection for training the model and predicting with it are shown in the left plot of Figure \ref{fig:pairing_prj}. Looking at Algorithm \ref{alg:Csawg}, 
we can see that step-size planning can be viewed as a step-size adaptation technique but with some distinctiveness in that the step-sizes are not adapted constantly but a planned diagonal step-size is applied from time to time, while a normal constant step-size is used in the main loop. This distinguishes step-size planning from existing step-size adaptation techniques such as Polyak step-size, Adam, IDBD and many others such as those discussed in our review, which adapt the step-size constantly every time step. 

Let us discuss the connection of step-size planning to planning in reinforcement learning, in particular, Dyna-style planning architecture in more details here. The first doubt when I came up with the word ``step-size planning'' was, {\bfseries \em Is this planning at all?} You are absolutely well justified to have this doubt: planning is for solving reinforcement learning problem and here Csawg optimizes a loss function with a special gradient descent process. How come Csawg is planning? In reinforcement learning and Markov Decision Processes (MDPs) \citep{sutton2018reinforcement,bertsekas2012dynamic,szepesvari2010algorithms}, the model has a natural definition, which is the transition probability matrix or kernel conditioned on an action, governing the transition dynamics between states or feature vectors. Take linear Dyna for example (see Figure \ref{fig:compare_lin_dyna}), the transition samples between successive feature vectors, $(\phi, \phi')$, are used to build a single-step model, in the form of a (full-sized) matrix, $F$. \footnote{One can treat the reward as a special feature in which case $b$ becomes one added row of $F$. So ignore $b$ for your understanding.}  At planning, a feature vector, $\phi$, is sampled from some distribution (e.g., the uniform distribution)
and projected to predict the expected next feature vector, based on which Temporal Difference (TD) error is calculated and an incremental algorithm such as TD without eligibility trace is applied. \citet{lin_dyna} established the consistency of planning using i.i.d sampling of feature vectors according to any distribution as long as the distribution has a nonzero probability of sampling each feature and the spectral radius of the transition model is guaranteed to be smaller than one, and demonstrated slightly faster policy evaluation (with mixed results in comparison to TD) than TD for a Dyna algorithm even aided with a prioritized sweeping of parameter updates \citep{moore1993prioritized}. Clear faster learning of linear Dyna was first shown by \citet{multi-step-dyna}, 
in which the key is to use (bigger) constant step-sizes instead of the diminishing step-size in planning because planning is not a stochastic approximation procedure like the online TD learning any more, but instead solving a model, and thus there is no need to suppress noises. 
It is noticeable that for policy evaluation, the transition model linear Dyna uses is the same as a key matrix in least-squares policy evaluation (LSPE) \citep{lspe96}. However, the use of the model in the planning procedure, especially the extension to action-conditioned models is nice in linear Dyna, which can be used to simulate experience, do off-policy learning and solve control problems while LSPE is only for policy evaluation. \citet{lspe03} established the convergence of using constant step-sizes for LSPE without using diminishing ones, which supports the choice of the planning step-size by \citet{multi-step-dyna}. \citeauthor{multi-step-dyna} also extended linear Dyna to multiple steps of planning, which lead to further accelerated learning for both policy evaluation and control without prioritized sweeping. 

Figure \ref{fig:compare_lin_dyna} is a side-by-side comparison of Csawg and linear Dyna. Both algorithms have three parts: learning (from online samples), modeling (from pairs of samples) and planning (from the estimated model). Architecture wise, they are strikingly similar although the problem context is different. Csawg estimates a multi-step model from the samples collected in the buffers while linear Dyna uses the immediate transition samples to build a single-step model. Linear Dyna predicts the expected next feature vector while Csawg predicts the expected weight in multiple iterations. 
Linear Dyna assumes linear function approximation while Csawg deals with a general function approximation. Because of the difference in problem context, learning in Csawg is performed by gradient descent while in liner Dyna it is by TD. This is a minor difference as both are incremental and online algorithms. 

\begin{figure}[t]
\centering
\includegraphics[width=\textwidth]{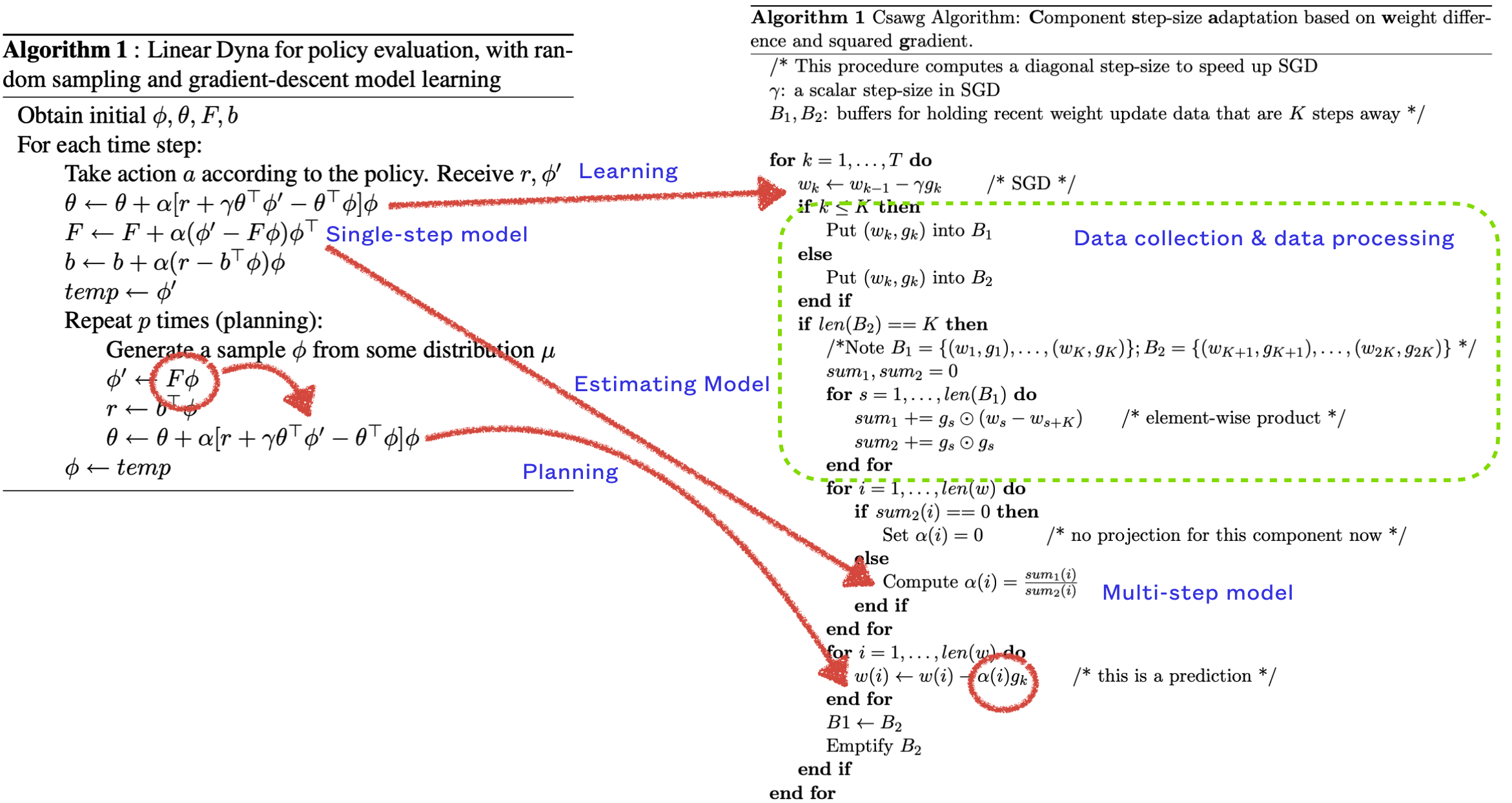}
\caption{Side-by-side comparison of Csawg and linear Dyna.
}
\label{fig:compare_lin_dyna}
\end{figure}

Building an empirical model from data is quite natural for reinforcement learning especially model-based algorithms. For example, besides the model used by linear Dyna, there are other models such as option-conditioned models \citep{sutton1999between}, RKHS MDP \citep{rkhs-mdps}, and pseudo-MDP \citep{yao2014pseudo}, etc.  
 Definitions and estimation methods of models may vary between model-based reinforcement learning algorithms, but they all arise from the inherent model with the MDPs. The situations with gradient descent based optimization in terms of problem formulation, analysis and experimentation have been quite different. In particular, {\em  we don't normally consider gradient descent as a state transition problem}, nor treat the step-size as a multi-step transition model. We don't also normally consider the
gradient descent update is {\em data}, nor use it because it is unknown whether this has any practical benefit.   

\begin{tcolorbox}
(E.) The perspective Csawg brings us is a new understanding of step-sizes, that the collection of them, in the form of diagonal matrix as we explored in the paper, can be viewed as a multi-step transition model, and we can learn them from the update data, take advantage of the learned step-sizes by projecting on our current gradient, and benefit from an accelerated convergence of gradient descent. 
\end{tcolorbox}

This new perspective essentially treats the gradient descent update as a process of generating data. Previously, Hyper-gradient uses the two successive gradients in the update to measure whether currently we are at the same side of the minimum. The IDBD algorithm uses historical gradients to measure the correlation with the current gradient which measures better in stochastic environments. However, this is still not enough. {\em Measuring where we are during the gradient descent} is the ultimate question for step-size adaptation. Only using the current gradient or even the historical gradients does not give a sufficient measure of where we are in the update. It will only be sufficient if we treat the gradient descent update as data. As we have more and more updates, the data reveals more information of where we are. Importantly, the data enables us to {\em learn step-size} and {\em plan with it}, which improves the way of parameter updating. This is key idea of Csawg, which distinguishes it from prior methods.

\section{Experiments}\label{sec:exp}
In this section, we conduct experiments by studying the performance of the Csawg algorithm for deterministic gradient descent.

\begin{figure}
\centering
\begin{subfigure}[c]{0.49\textwidth}
\centering
    \includegraphics[width=3.3in]{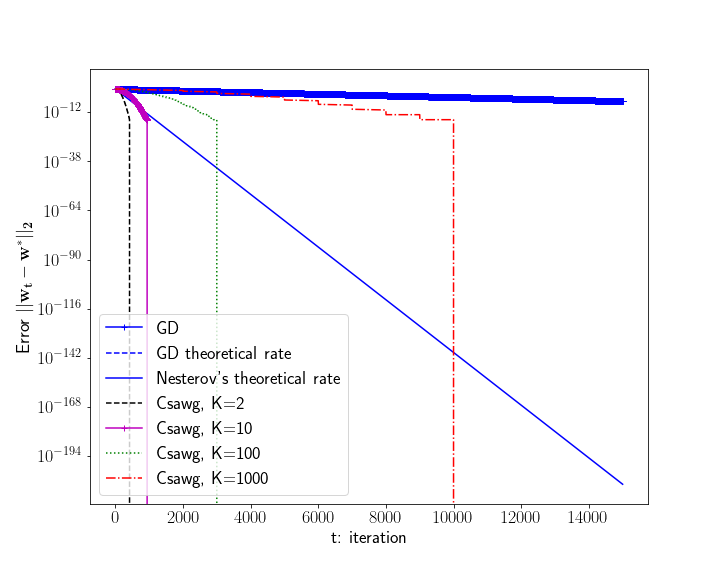}
    \caption{Convergence rate comparison.}
    \label{fig:q2_algs}
\end{subfigure}
\hfill
\begin{subfigure}[c]{0.49\textwidth}
\centering
    \includegraphics[width=3.2in]{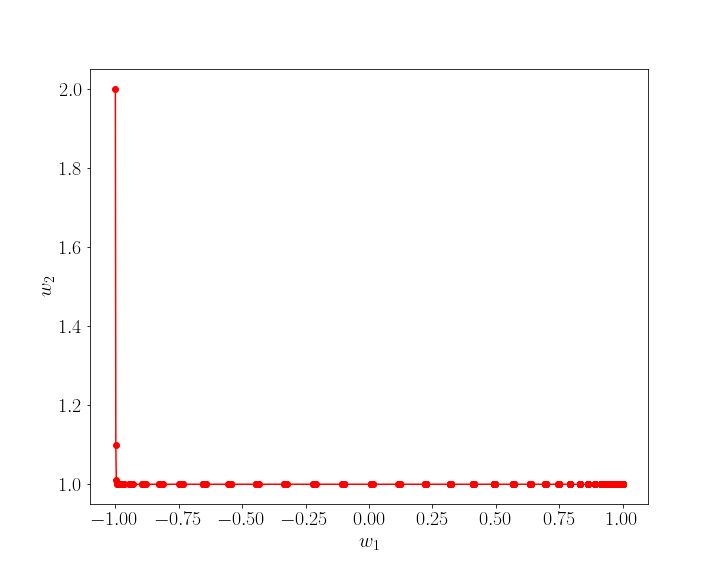}
    \caption{Update trace of Csawg. }
    \label{fig:q2_trace}
\end{subfigure}
\caption{Convex function. (a): Convergence rate of GD, Nesterov's accelerated gradient, and Csawg algorithms. (b): Update trace by Csawg K10 for the convex problem. The dense dots are the updates by the online GD and the extended straight lines that connect them are the updates by planning, which bootstraps in the direction of $w^*$. 
}
\label{fig:q2_err_trace}
\end{figure}

{\bfseries Convex}. 
The first problem is a convex function, $f(w) = \frac{1}{2}(w-w^*)^T Q (w-w^*) $, where $Q=\begin{pmatrix}
1000& 0\\
0 & 1
\end{pmatrix}
$, and $w^*=[1, 1]$. 
All compared algorithm started from $w_0=[-1, 2]$, and their performances are shown in Figure \ref{fig:q2_algs}. The theoretical rate of GD for this problem is $1- \frac{\mu}{L}$, where $\mu=1$ (the minimum eigenvalue of $Q$) and $L=1000$ (the maximum eigenvalue). This is achieved with a step-size smaller than but close to $1/L$. The GD algorithm used a step-size $0.00099$, which matches this rate closely, as shown by the overlapped blue lines at the top of the plot. The problem is ill-conditioned and the convergence rate of GD is very slow, which is $O(0.999^t)$.
The convergence rate of Nesterov's accelerated gradient and heavy ball is $1- \sqrt{\frac{\mu}{L}}$ \citep{nesterov2003introductory}, which is about $O(0.968^t)$ in this case, much faster than GD. 

Our Csawg used a step-size $0.0009$ for its GD, which is smaller and less optimal than GD. The performances of using different $K$s are studied and shown in the figure. 
Interestingly, Csawg $K=2$ (shortened for Csawg K2) performed the best, reducing the error to zero in less than 500 iterations. 
Csawg K10, K100, K1000 eventually also reduced the error to zero but they are slower than Nesterov's accelerated gradient in the beginning of the iterations.
The sharp dipping in the errors of Csawg algorithms indicates that the algorithm's convergence rate ``may be'' orders higher than Nesterov's accelerated gradient and heavy ball method.

Figure \ref{fig:q2_alpha} shows the step-sizes learned by Csawg for different $K$. Some interesting observations are as follows. (1) For component one, the step-size increases first and then decreases (sharply) to zero.
(2) The peak of all the step-sizes for component one is {\em exactly} one eventually, which is super interesting. 
The moment that the step-size being one means that the gradient is in the right direction towards $w^*$ and with a perfect magnitude just matching the current distance to the optimal solution. That is, $g_k(i) = w_k(i) -w^*(i)$.  
(3) What is intuitive reason that the algorithm is fast? Figure \ref{fig:q2_trace} shows the update trace of Csawg K10. 
It shows the first component is hard to learn; and Csawg helps by its step-size that transitions multiple steps ahead towards $w^*$, especially the first component. The update is surprisingly smooth and fast --- for this problem, many step-size adaptation algorithms especially a single step-size shared for both components would expect oscillatory behaviors or slow learning. 
(4) The step-size adaptation is component dependent. 
There wasn't much adaptivity in the second component, from the second plot of Figure \ref{fig:q2_alpha}. 
This interesting selective behavior by the step-size shows that Csawg is suitable for detecting learning difficulty of parameters in learning. Sparse and pulse signals in component step-sizes learned by Csawg like shown in Figure \ref{fig:q2_alpha2} is an indication that the learning of this component is easy, while a dense step-size signal like Figure \ref{fig:q2_alpha1} indicate that the learning of the component needs lots of adjustments in step-sizing and the learning is hard.  

{\bfseries Non-convex}. 
Rosenbrock function is a challenging problem for gradient descent because the minimum hides in a long and narrow valley, which is created by the large ratio between the coefficients of the two composite terms:
\[
f(w) = (w_1 -1)^2 + 100 (w_2- w_1^2)^2.
\]
We run gradient descent (GD) and Csawg, both starting from $w_0= [-1, 0]$. 
The results are shown in Figure \ref{fig:rosenbrock}. 
Interestingly, $K=2$ leads to the fastest convergence for Csawg, even faster than $K=5$. With our derivation, one may expect that Csawg K2 will be two times faster than GD because it only looks two steps ahead. However, it appears that the algorithm is much faster than that (see the caption of the figure for details). This is due to that Csawg K2 uses a short history of the update and the step-size is learned from the updates moments ago, which gives highly accurate predictions for temporal neighborhoods of the update. This extremely short history helps because the update close to the convergence point ($w^*=[1, 1]$) easily shoots to the other side of the valley. 

\begin{figure}
\centering
\begin{subfigure}[c]{0.43\textwidth}
\centering
    \includegraphics[width=2.8in]{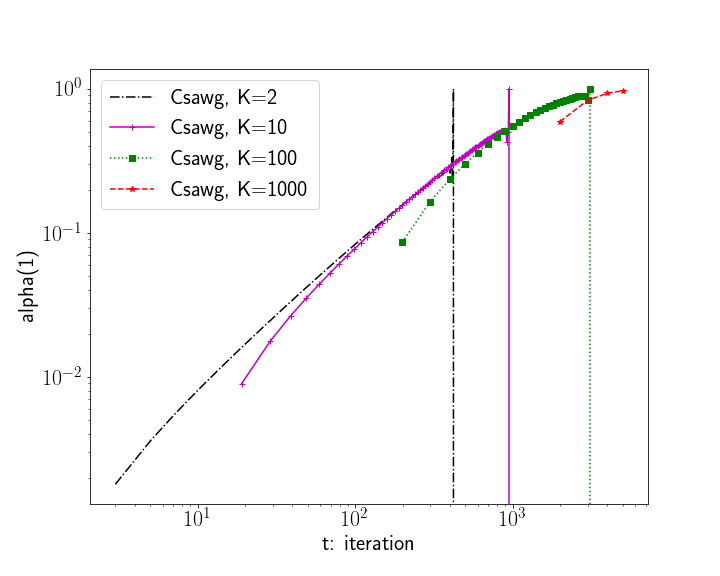}
    \caption{$\alpha(1)$}
    \label{fig:q2_alpha1}
\end{subfigure}
\hfill
\begin{subfigure}[c]{0.43\textwidth}
\centering
    \includegraphics[width=2.8in]{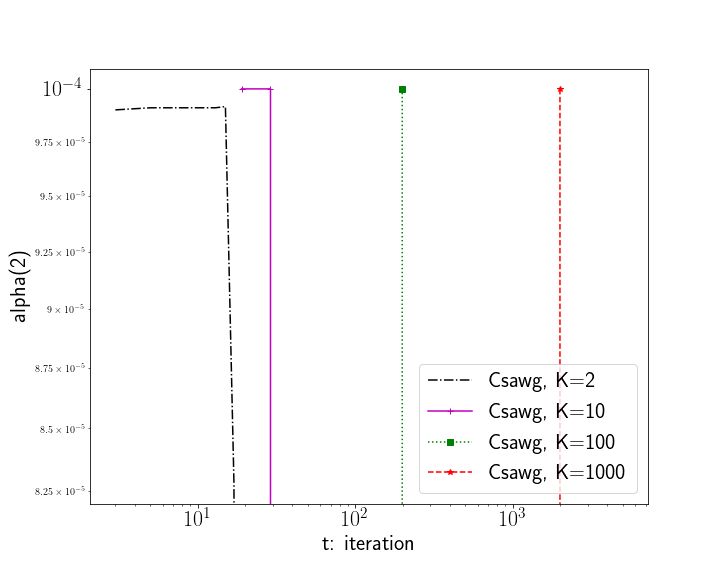}
    \caption{$\alpha(2)$}
    \label{fig:q2_alpha2}
\end{subfigure}
    \caption{Step-sizes learned by Csawg for the convex function. For $\alpha(1)$, all Csawg $K$ algorithms have an increasing stage.
    Notably, Csawg K2, K10 and K100 increased to exactly 1, and Csawg K1000 increased to about 0.97 in the figure. Give more iterations for Csawg K1000, $\alpha(1)$ also increased to exactly 1. 
    All the step-sizes eventually dropped to zero (not shown because of log in the y-axis), and this happened at the sharply dropping moments shown by the plot. 
    }
    \label{fig:q2_alpha}
\end{figure}

Figure \ref{fig:rosenbrock_csawg_alpha2} shows how one component step-size, $\alpha(2)$ by Csawg
, changes during iterations. For Csawg K5, the step-size converges in the end. While for Csawg K2 and K10, the step-size does not have a convergence point. 
However, Csawg K2 shows a faster convergence rate.
This was counter intuitive to us at first. 
Having a convergence point may appear to be more advantageous because the final update can be more stable. Although the behavior of the step-size looks very dynamic and oscillatory, it does not prevent a smooth and consistent convergence for $w$. Furthermore, note that there is a significant percentage of times where the step-size is negative. Figure \ref{fig:rb_traces} shows the update traces by GD and Csawg. Csawg has a bigger step between the points, showing bootstrapping is effective. The trace of Csawg shows that occasionally planning can bring the update a bit off track, but GD brings it close back. Thus this shows online GD is an important element of Csawg.

\begin{figure}[t]
\centering
\begin{subfigure}[c]{0.43\textwidth}
    \centering
    \includegraphics[width=2.7in]{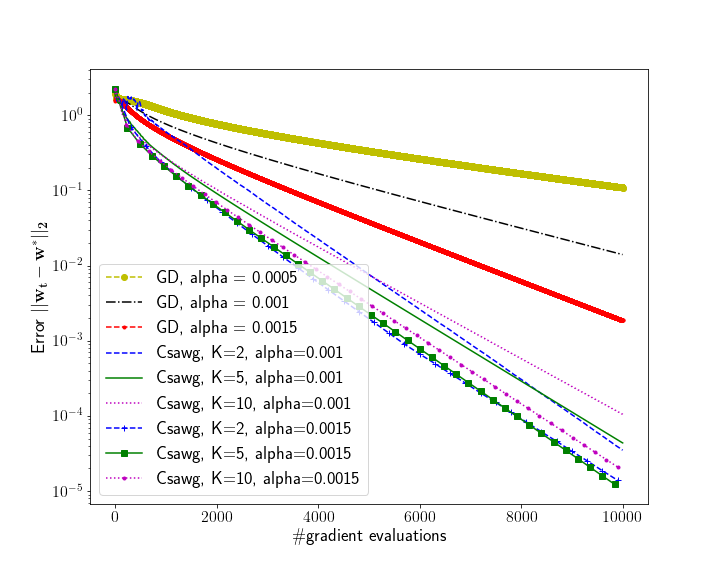}
    \caption{Convergence rates}
    \label{fig:rosenbrock}
\end{subfigure}
\hfill
\begin{subfigure}[c]{0.43\textwidth}
    \centering
    \includegraphics[width=2.7in]{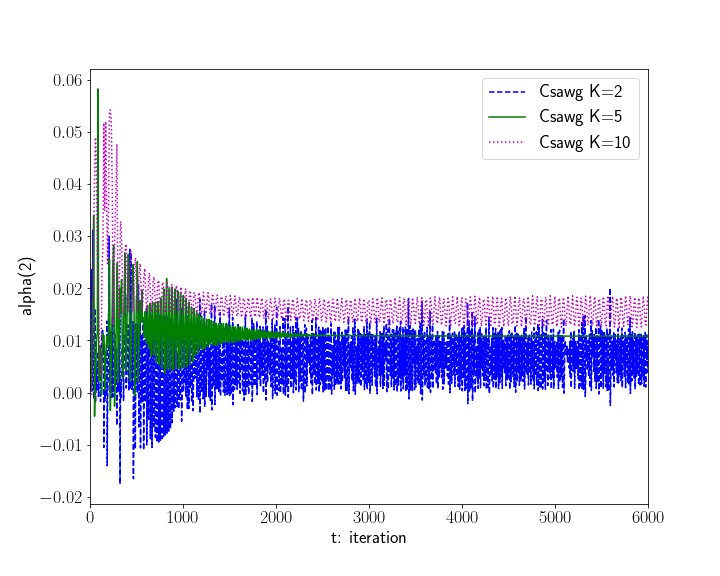}
    \caption{Step-size $\alpha(2)$ learned by Csawg.}
    \label{fig:rosenbrock_csawg_alpha2}
\end{subfigure}
\caption{Rosenbrock function. (a) Convergence rate comparing Csawg vs. gradient descent (GD). 
    The step-size for GD was selected from $\{0.0005, 0.001, 0.0015, 0.002\}$, and $0.001$ performed the fastest and stably for the first 20,000 iterations.
    For Csawg, we compared using $\alpha=0.001$ and $\alpha=0.0015$ for its GD. The plot shows that even Csawg with un-optimal step-size $\alpha=0.001$ for its GD performed significantly faster than the GD with the best step-size for all $K=2, 5, 10$. For 10,000 gradient evaluations \protect\footnotemark,
     the speedup over GD with the best step-size is, $400/320/133$ times for $K=2/5/10$. The speedup was measured by the ratio between the error of GD and that of Csawg. The compared two algorithms with this ratio have the same number of gradient evaluations in computing the ratio. (b) Step-size $\alpha(2)$, which is learned by Csawg K2, K5 and K10. Note the many moments of negative step-sizes for K2. 
}
\label{fig:rb_err_alpha}
\end{figure}

\footnotetext{For GD, the number of gradient evaluations is the same as the number of iterations. For Csawg, the number of iterations is smaller because there is also gradient evaluation in planning.}

\begin{figure}
\centering
\begin{subfigure}[c]{0.43\textwidth}
\centering
    \includegraphics[width=2.5in]{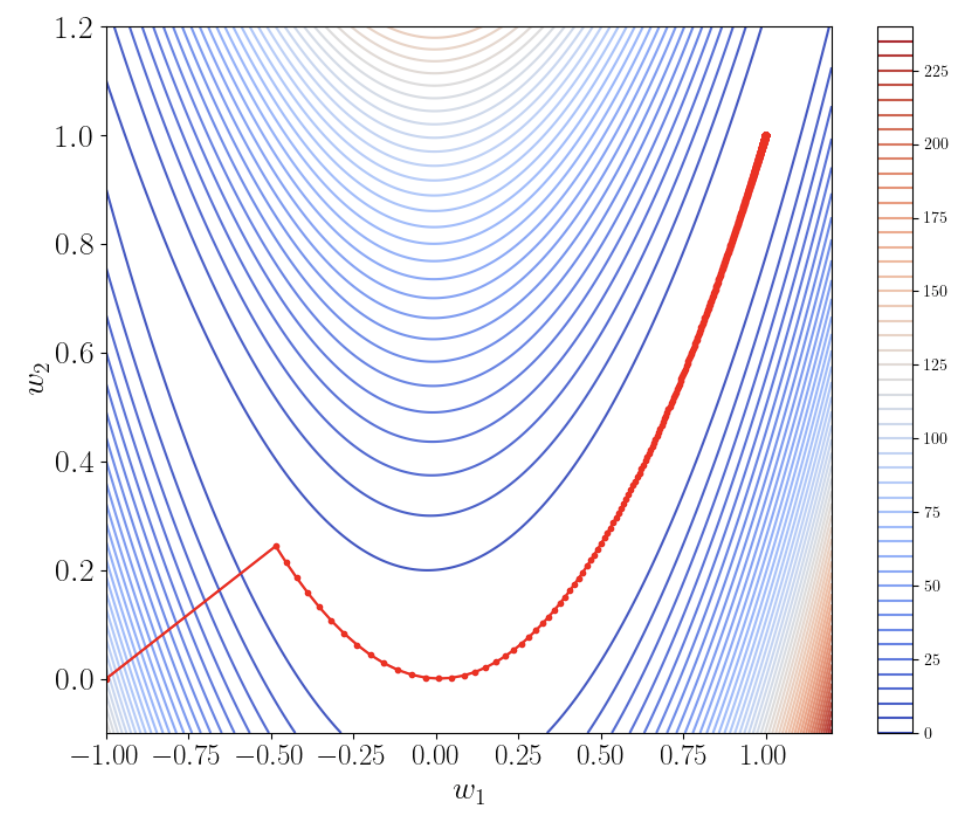}
    \caption{GD update trace}
    \label{fig:rb_gd_trace}
\end{subfigure}
\hfill
\begin{subfigure}[c]{0.43\textwidth}
\centering
    \includegraphics[width=2.5in]{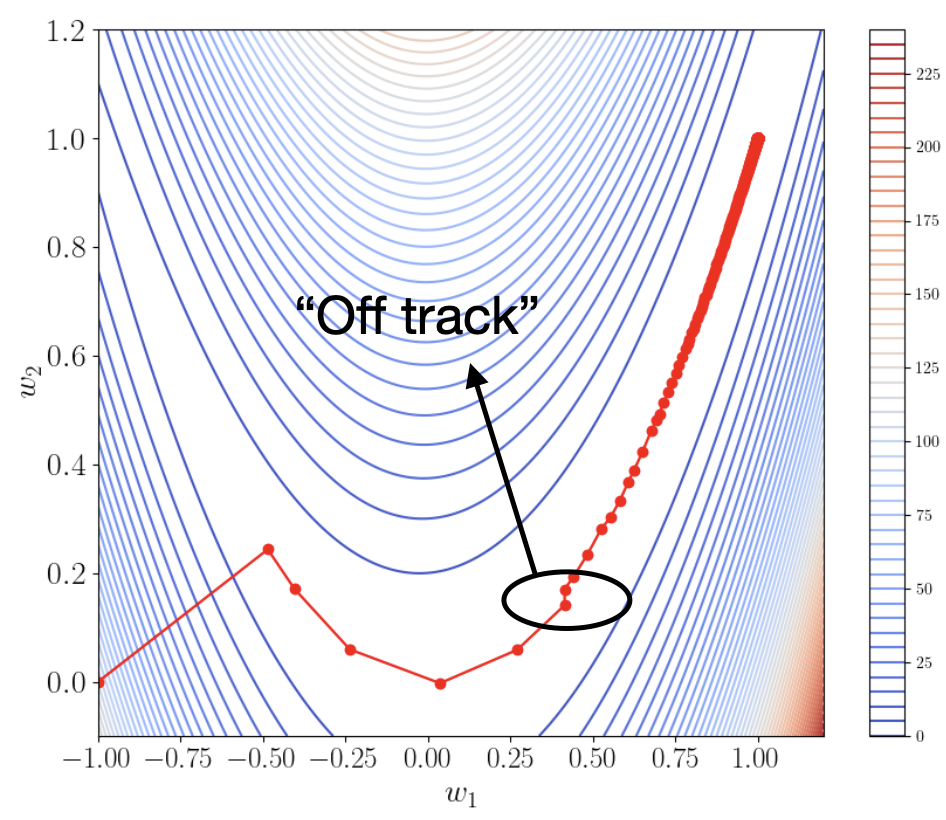}
    \caption{Csawg K10 update trace.}
    \label{fig:rb_trace_K10}
\end{subfigure}
    \caption{Update traces on Rosenbrock function, shown is a sparse plot with dots being 20 iterations away for both algorithms. The right plot shows Csawg can occasionally update away from the desired path. 
    }
    \label{fig:rb_traces}
\end{figure}

\section{Repeated Step-size Planning} \label{sec:multi-step}
If we have a good transition model, we should be able to 
apply it multiple times at one moment to gain even more speedups. While this appears to be a natural idea, it does not work well in a straightforward sense. The ``off-track'' behavior in Figure \ref{fig:rb_trace_K10} magnifies when applying a step-size model for repeated planning. 
Figure \ref{fig:pairing_prj} middle plot shows at the third time when planning is applied, the step-size mainly moves the weight projection to the right, since this is the direction in which the update is so far (red dots are mainly going to right until this time). The green big green arrow shows the first time applying the step-size for projection in planning, which however goes out of the valley. 
To bring the update back to track and stabilize training, we apply a number of SGD updates in each step of planning. 
This is shown in Algorithm \ref{alg:Csawg-repeated}.

Back to Figure \ref{fig:pairing_prj} middle plot, this shows after the application of the projection indicated by the big green arrow, from the projected weight, there are a number of GD updates (magenta crosses) applied, which indeed bring the update back to the valley, and correct the going-to-the-right direction as learned by the current step-size. Four more planning steps are applied in the plot, after each of which the GD updates bring them back to track. 
At this single iteration, the five repetitions of planning with a total of 55 gradient evaluations \footnote{There were ten GD updates in each planning step used in this experiment. So the total number of gradient evaluations is 5 step-size model projections plus 50 GD updates.} lead the weight great progresses in the valley that would otherwise taking gradient descent thousands of gradient evaluations. 
After this time step, the planning in the middle plot is finished, and the weight in the end is passed to the online GD update of the next time step and the update experience is accumulated again. 

The next time planning is applied as shown in the right plot of the same figure, the step-size is updated with the new update experience and now both step-size components are significantly positive, leading the application of the weight  projection towards the upward and right direction where the minimum is located, giving five projected weights (black squares) each followed by a number of GD updates. 
The bringing-back GD update fine tunes the step-size projection in this case (as indicated by the closeness of crosses and squares), and it brings the update closer to track.

\begin{figure}[t]
\centering
\begin{subfigure}[t]{0.49\textwidth}
\centering
    \includegraphics[width=3.0in]{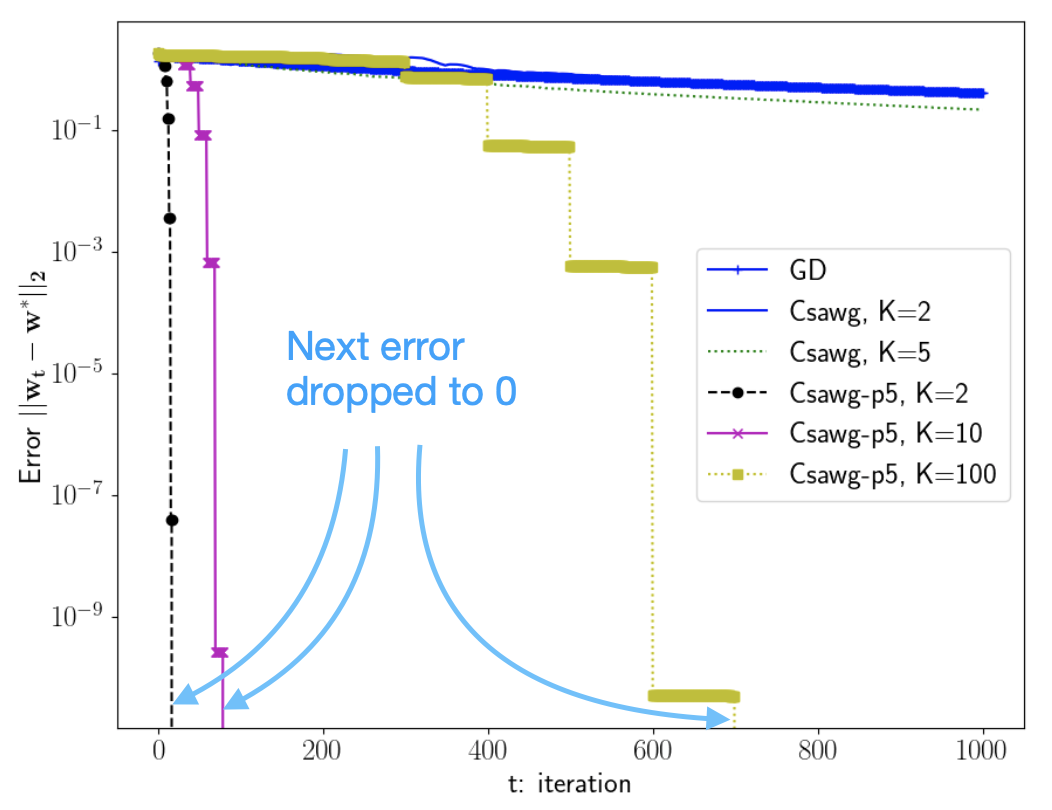}
    \caption{Convergence rates. }
    \label{fig:rb_p5}
\end{subfigure}
\hfill
\begin{subfigure}[t]{0.44\textwidth}
    \centering
    \includegraphics[width=3.2in]{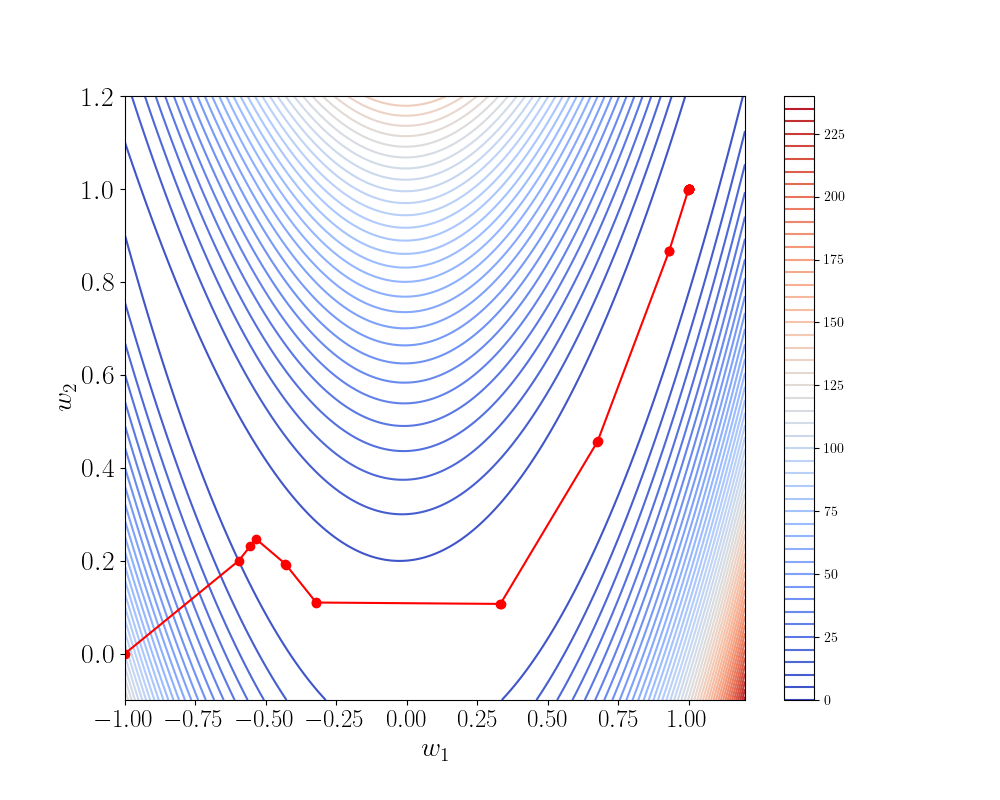}
    \caption{Update trace of Csawg-p5 K2.
    }
    \label{fig:rb_p5_trace}
\end{subfigure}
\caption{Csawg planning five steps for Rosenbrock function. (a) shows this further speeds up Csawg. 
     Csawg-p5 K10 reached zero error  at iteration $79$ with  $465$ total gradient evaluations (including in both online SGD and planning updates); 
    Csawg-p5 K2 reached zero error at iteration $17$ with $458$ total gradient evaluations.
    (b) Update trace of Csawg K2 planning five steps, shown data points are plotted every iteration for the weights in the main loop.
}
\label{fig:rb_traces_p5}
\end{figure}

The introduction of bringing-back GD updates in the planing procedure is fruitful. 
It not only corrects the ``off-track'' behavior but also stabilizes repeated planning and leads to even faster convergence. Figure \ref{fig:rb_p5}
shows five steps of planning for Csawg with different $K$. 
Multiple steps of planning tremendously accelerates Csawg. For example, with five steps of planning Csawg K10 brought the error down to zero at iteration $79$, with only $465$ total gradient evaluations (including in both online GD and planning updates); 
Csawg K2 reached zero error at iteration $17$ with $458$ gradient evaluations. 
The plot is the first $1,000$ iterations, for which Csawg's speed advantage over GD has not shown up yet --- refer to Figure \ref{fig:rosenbrock} to check for the speedup showing up at later iterations. 
The errors have a ``going-down-the-stairs'' pattern, which is especially clear from Csawg K100: the first time planning is applied is at iteration 200 to wait for the two buffers to be ready, which does not reduce the error much; then the second and the following three times when planning is applied, the error goes down like steep stairs. Note the ``flat stairs'' in the plot corresponds to online GD updates, showing slowness in reducing error but they are collected as experience based on which Csawg learns its step-size transition model and plan multiple times to bootstrap effectively. 

The step-sizes learned by Csawg with multiple steps of planning can exceed one, as shown in Figure \ref{fig:rb_alpha_p5}. Interestingly, the peaks of both step-sizes by different $K$ are all close to $2.5$. The step-sizes show an increasing trend in the beginning iterations and then quickly drop to zero. The moment that the step-size dropping to zero is almost the same time of reaching $w^*$.
With multiple steps of planning, even small $K$ is able to make large leaps in the complex loss contour towards the minimum, as shown in Figure \ref{fig:rb_p5_trace} for Csawg K2 --- compare this with the single-step planning in Figure \ref{fig:rb_trace_K10}. 
The two component step-sizes have a similar learning pattern for this problem, because the optimal direction is along the valley which is upward and rightward. In the beginning few planning steps, the second step-size can be negative; this makes sense because in the beginning updates the movement in the second axis changes the direction, which causes a negative step-size saying that ``Don't go in the direction as the current gradient descent tells you, go ascent. The opposite''. This may sound odd as the theory of gradient descent tells us, but the step-size is learned from experience, and the direction it gives is more trustable than simply according to the temporal gradient.

\begin{figure}
\centering
\begin{subfigure}[c]{0.49\textwidth}
\centering
    \includegraphics[width=2.9in]{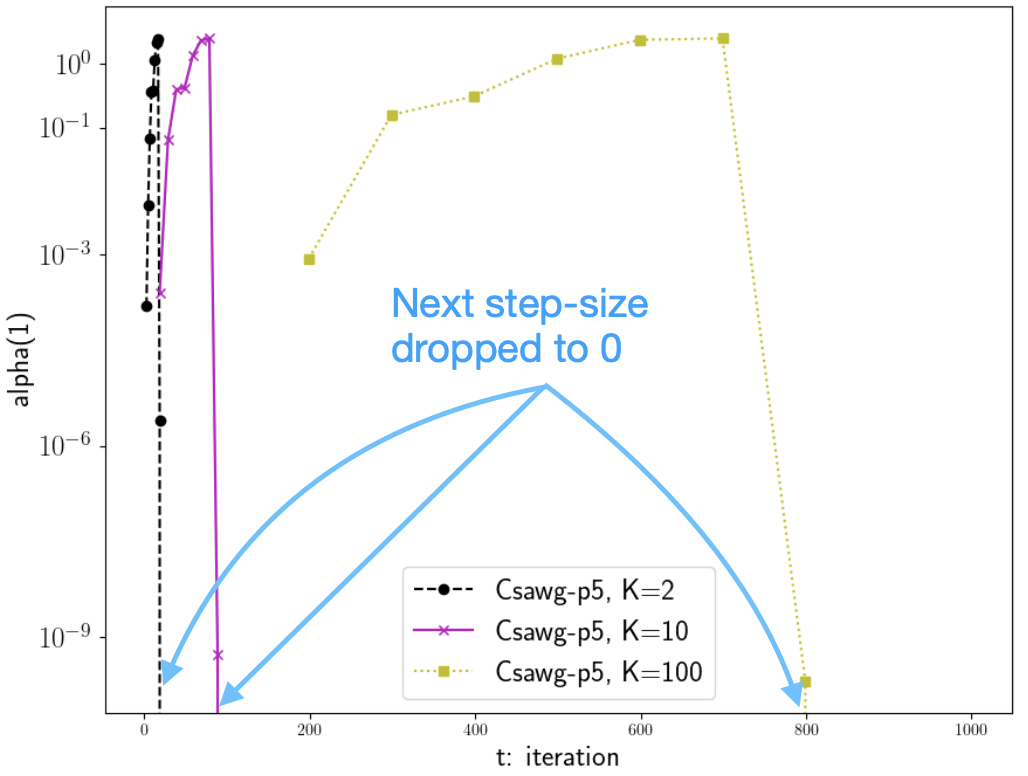}
    \caption{$\alpha(1)$}
    \label{fig:rb_p5_alpha1}
\end{subfigure}
\hfill
\begin{subfigure}[c]{0.49\textwidth}
\centering
    \includegraphics[width=2.9in]{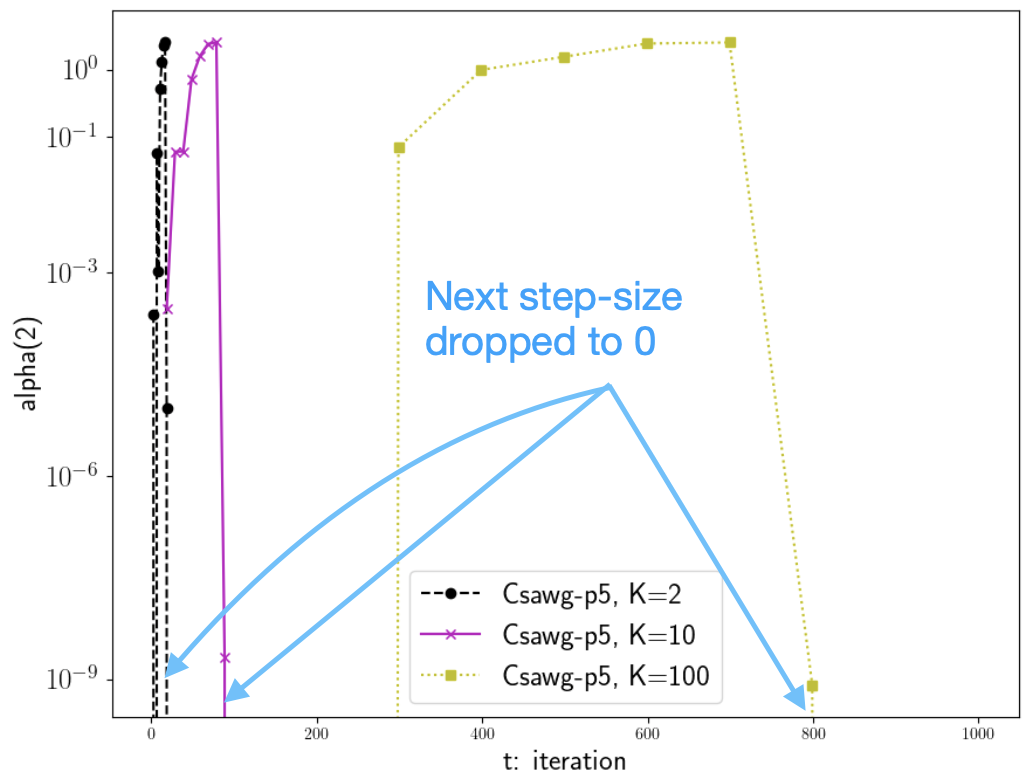}
    \caption{$\alpha(2)$}
    \label{fig:rb_p5_alpha2}
\end{subfigure}
    \caption{Step-sizes learned by Csawg with five planning steps for Rosenbrock function.
    }
    \label{fig:rb_alpha_p5}
\end{figure}

How does ADAM perform for Rosenbrock function?
ADAM has two algorithmic ideas combined: momentum and gradient normalization, according to our review. 
To see which idea(s) helps, we compared the {\em heavy ball} algorithm (equation 9 of \citet{polyak1964ball}):
\begin{equation}\label{eq:momentum_classical}
w_{k+1}  = w_k - \alpha f'(w_k) + p \Delta_k,  
\end{equation}
where $\Delta_k = w_k - w_{k-1}$. 
This is the basic form where many momentum techniques in neural networks derive from, e.g., ADAM's incremental computation of momentum in averaging gradient form \citep{polyak_averaging}, e.g., see our review in Section \ref{sec:adam}.
We also compared with RMSprop, which normalized the gradient with the square root of the averaged gradient squares over time.

\begin{algorithm}[t]
\begin{algorithmic}
\State /* This procedure computes a diagonal step-size to speed up SGD 
\State $\gamma$: a scalar step-size in SGD 
\State $B_1, B_2$: buffers for holding recent weight update data that are $K$ steps away */
\State 
\For{$k=1, \ldots, T$}
    \State $w_k \gets w_{k-1} - \gamma g_k$  \quad \quad /* SGD */
    \If{$k \le  K$}
        \State Put $(w_k, g_k)$ into $B_1$ 
    \Else 
        \State Put $(w_k, g_k)$ into $B_2$ 
    \EndIf
    
    \If{$len(B_2) == K$} \quad \quad 
    \State /*Note $B_1=\{(w_1, g_1), \ldots, (w_{K}, g_{K})\}; B_2=\{(w_{K+1}, g_{K+1}), \ldots, (w_{2K}, g_{2K})\}$ */
        \State $sum_{1}, sum_{2} = 0$
        \For{$s =1, \ldots, len(B_1)$}
            \State $sum_{1} += g_s \odot (w_s - w_{s+K} )$\quad \quad /*  element-wise product */
            \State $sum_{2} += g_s \odot g_s$ 
        \EndFor
        \For{$i =1, \ldots, len(w)$}
            \State Compute $\alpha(i) = \frac{sum_{1}(i)}{sum_{2}(i)}$
            
        \EndFor
        
    \For{$p=1, \ldots, P$}        \quad \quad /* repeated planning  */
        \For{$i=1, \ldots, len(w)$}
            \State $w(i) \gets w(i) - \alpha(i) g_k$ 
        \EndFor
        \For{$m=1,\ldots, M $}
            \State $w \gets w - \gamma g_k$ 
        \EndFor
    \EndFor
    
    \State $B1 \gets B_2$
    \State Emptify $B_2$ 
        
    \EndIf

\EndFor

\end{algorithmic}
\caption{Csawg repeated planning.}
\label{alg:Csawg-repeated}
\end{algorithm}

Let's foresee possible outcomes of this experiment before getting started. In the set of benchmark algorithms, RMSprop tells us if the normalized gradient technique can speed up gradient descent or not in this scenario.
The heavy ball method tells us if the momentum technique helps or not. ADAM is probably going to be faster than either of these two techniques if both help; or one of these two techniques work then ADAM should also benefit from it; or neither of the two techniques work, which should match that ADAM would be ineffective in this scenario either. 

We searched extensively for the performance of these algorithms. 
For RMSprop, we searched $(\alpha, \beta)$ over $\alpha\in \{0.0005, 0.001, 0.0015, 0.01\}$ and $\beta\in \{0.8, 0.9, 0.99\}$. 
For the heavy ball algorithm, we searched $(\alpha, p)$ with $\alpha=\{0.0015\}$ and $p\in \{0.8, 0.9\}$.
For ADAM, we searched $(\alpha, \beta_1, \beta_2) $ over $\alpha \in \{0.005, 0.01 \}$, $\beta_1 \in \{0.9, 0.99, 0.999\}$ and $\beta_2\in \{0.99, 0.999, 0.9999 \}$. Note this covers the default beta parameters used by ADAM \citep{kingma2017adam}. 
These candidate search sets were manually selected through repeated trials.

\begin{figure}[t]
\centering
\includegraphics[width=5.8in]{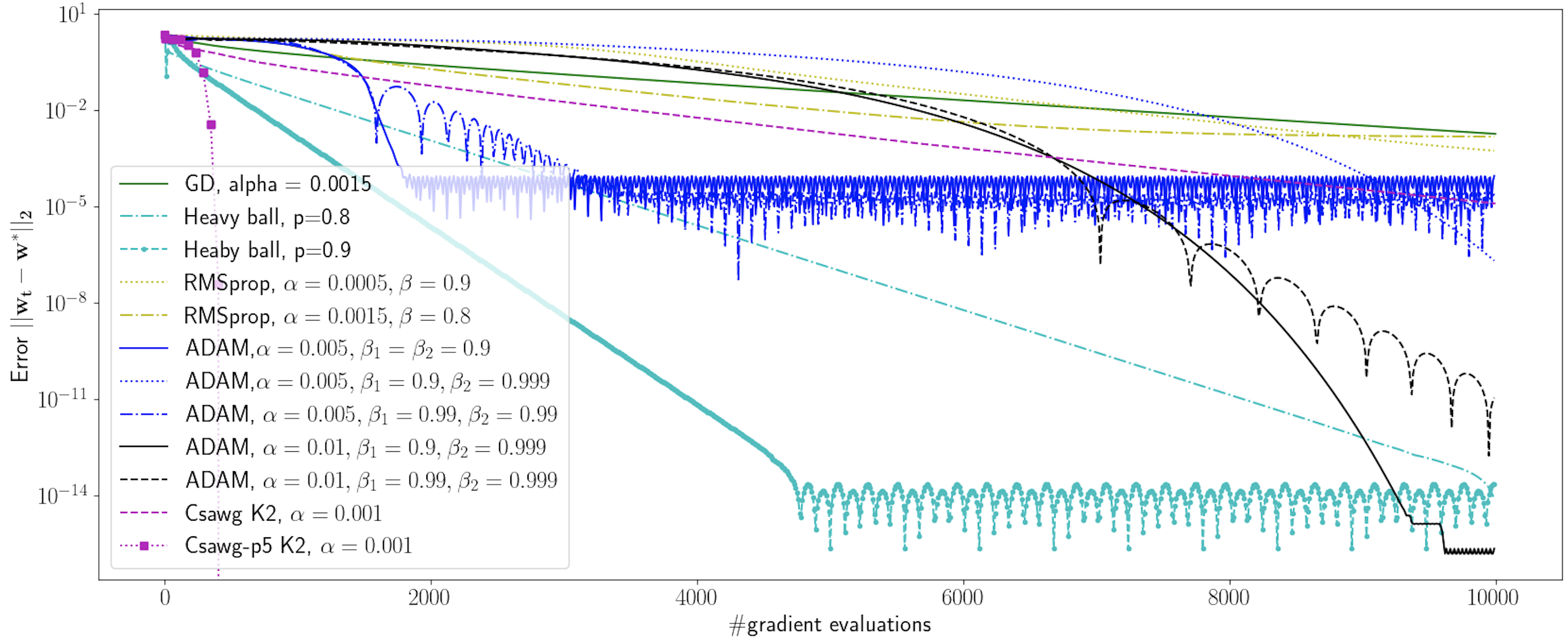}
\caption{Comparing Csawg with momentum and gradient normalization GD algorithms: Rosenbrock function.
}
\label{fig:rb_adam}
\end{figure}

Heavy ball turns out to be very effective for this problem of finding the minimum in the narrow valley; while normalized gradient is not. In Figure \ref{fig:rb_adam}, 
the cyan lines near the bottom show heavy ball algorithms are much faster than GD and Csawg with single-step planning; though a large momentum rate $p$ inevitably leads to oscillation around the minimum in the end. 

The error curves of RMSprop (yellow lines) have an intersection around 9,000 iterations, showing there is a tradeoff between learning speed and quality. 
One RMSprop was faster than GD, but the speedup is insignificant. 

ADAM lines (in blue and black) are in the middle of these two category of algorithms. 
First noticeable is that two blue lines are faster than GD after a bit fewer than 2,000 iterations but they have a similar oscillation pattern to the heavy ball method. This is due to the momentum term effect near the minimum. The two black lines using a larger $\alpha$ than the blue lines had a better solution in the end but they are faster than GD only after about 6,000 iterations.  

Second, the oscillation of ADAM happens at a higher error level in the two blue lines than heavy ball. This is due to the normalization of the gradient. 
ADAM did not have a fast convergence similar to heavy ball because unlike momentum, gradient normalization cannot be easily turned off in ADAM. By controlling $\beta_2$ to be small or large, the normalization of the gradient is always there --- the difference is just putting more or less weight on the recent gradient squares in the normalizer. 
The only way to limit its effect is through a large $\beta_2$, e.g., the default $\beta_2 = 0.999$, which is shown by two black lines and the dotted blue line.
Unfortunately, in the meanwhile this reduces the effect of the latest gradient as well, which causes slow learning in this scenario. 
Thus this example shows that when heavy ball helps speed up convergence while gradient normalization does not in this case of narrow-valley minima, ADAM can be slower than heavy ball. 

Csawg algorithm with multiple steps of planning is even faster than heavy ball. 
Comparing to heavy ball, Csawg algorithm has a faster convergence after a few beginning iterations, this is due to multiple steps of planing using a learned step-size model; the algorithm did not have an oscillation like heavy ball, showing that the algorithm's step-size has a proper way of detecting the closing-to-minimum situation and drives towards zero accordingly at the right time.

\section{Conclusion and Future Work}\label{sec:conclusion}
In this paper, we propose {\em Csawg}, a novel class of methods for accelerating the convergence of gradient descent. The methods use the update experience of gradient descent to learn a diagonal step-size, based on which the methods bootstrap and accelerate. We have five novel conclusions for gradient descent that are discussed in the boxed texts in the paper, which we briefly summarize here:

\begin{enumerate}[(A.)]
    \item The meta-gradient method adapting a single step-size cannot converge faster than a linear rate, and it suffers from ill conditioning of the problem. 
    
    \item Negative step-sizes makes sense for both deterministic and stochastic gradient descent. This is very surprising given the long history of using non-negative step-sizes in many areas of modern computer science. 
    
    \item The diagonal-matrix step-size has the same transformation power as the full-sized matrix (for non-degenerate vectors). This is in a great contrast to linear transformation theory which requires a full-sized matrix to transform vectors to the whole parameter space. 
    
    \item The wide practice of using a scalar step-size for training gradient descent is not good for fast convergence because it is over-stretched. One should use an individual step-size for each parameter. 
    
    \item Csawg brings us a new perspective of the step-size, in that the collection of the individual step-size for each parameter is a multi-step model that can be learned from the update data. This is also the first time introducing the data view of gradient descent update in the literature.
    
\end{enumerate}

We illustrate the performance of the new methods on a convex problem and a non-convex problem. On the convex problem, which is ill-conditioned, the methods converge faster than GD and Nesterov's accelerated gradient, and beyond the Nesterov's theoretical rate. On the famous non-convex Rosenbrock function \citep{10.1093/comjnl/3.3.175}, the methods continue to converge much faster than GD, and the repeated-planning versions of the methods spend fewer than 500 {\em gradient evaluations} for zero error. 
Note in optimization, it is well recognized that GD is slow when problems are ill-conditioned or non-convex, in which case it is much slower than heuristic search methods. 
In particular, GD spends about 10000 iterations to reach $10^{-3}$ accuracy for Rosenbrock function. 
The Nelder-Mead method \citep{Nelder1965ASM}, which is a heuristic search optimization method based on evaluating the function on a few dynamically adjusted vertices in a polytope, achieves $10^{-10}$ error after 185 function evaluations. Adaptive coordinate descent achieves the same error with about 325 function evaluations \citep{loshchilov2011adaptive}. 
Thus if we ignore the difference between gradient evaluations and function evaluations, in this paper we demonstrated that Csawg, a seemingly ``first-order'' GD algorithm, achieves similar rate to these methods. 
Note that gradient evaluations are much more complex than function evaluations in deep neural networks. Thus we expect step-size planning can save more computation by accelerating SGD for deep learning. 

Establishing convergence and convergence rate is an interesting future work. 
The convergence rate of ADAM and Adagrad (similar to RMSProp but with $\beta_2=1$) are $O(log t/t)$ in terms of the expected gradient norm \citep{bach_adam}. The convergence of ADAM comes from establishing that the correlation between the update by ADAM and the true gradient descent is larger than a term that depends on the norm of the scaled true gradient, suggesting in the long run the update goes along the true gradient descent direction, see Lemma 5.1 of \citep{bach_adam}. The bound comes from a growth limit of the expected sum of updates over iterations for each component, see e.g. their Theorem 2 and Lemma 5.2. 
This rate is for loss function whose gradient is smooth and bounded. If we further assume strongly convexity of the loss function, then SGD can converge with a linear convergence rate for the expected error in the weight, which can be achieved by stochastic Polyak step-size \citep{sps_polyak}. Their rate is characterized by $O(1-\frac{\mu}{L})^t$, where $\mu, L$ are the strong-convexity of $F$ and L-smooth factor of $F$ (i.e, Lipschitz constant of the gradient $F'$). The two rates have a big gap though the second one is established with a stronger assumption on the loss function.
Furthermore, for the strongly convex and L-smooth loss functions, the convergence rate of SGD may still be improved. 
In the deterministic setting, Polyak step-size is optimal for convex function whose gradient is non-smooth. However, when the loss function is both strongly convex and L-smooth, Nesterov's accelerated gradient (AGD) has a faster convergence rate, which is $1-\sqrt{\frac{\mu}{L}}$ \citep{nesterov2003introductory}.
It is well known that this convergence rate is tight for first-order methods \citep{nemirovskij1983problem,arjevani2016lower} in the sense of the required number of function evaluations to reach an accuracy. 
There is also a recent method that achieves this optimal rate with line search \citep{bubeck2015geometric} whose proof is much simpler and easier to understand than AGD. 
There is a recent work by \citet{PL_nonuniform} that extends the L-smooth and PL condition with local constants instead of the global $L$ and $\mu$, and they showed that this extension enables a tailored gradient descent algorithm which normalizes the gradient by the local $L$ constant to converge faster than $O(1/k^2)$. 

First-order methods are defined as
producing iterates satisfying
\[
w_k \in w_0 + \mbox{span}\{g_0, \ldots, g_{k-1}\}, 
\]
where $g_k$ is the gradient of the loss at $w_k$, for each $k$. From this perspective, although Csawg  uses only the gradient (which is first-order) and the computation involved is $O(d)$, it is not a first-order method because the gradients are transformed to form the basis of $w_k$:
\[
w_k \in w_0 + \mbox{span}\{g_0, \ldots, g_{k_1-1}, \alpha_{k_1-1}g_{k_1-1}, 
\ldots, g_{k_2-1}, \alpha_{k_2-1}g_{k_2-1}, \dots, 
\}, 
\]
where $k_i$ is a time step that planning is applied. 
For repeated planning, there will be $p$ basis vectors, from powers of the diagonal step-size every time planning is applied, $\{\alpha_{k_i-1}^{s}g_{k_i-1}| s=1, 2, \ldots, p\}$.  
Thus Csawg raises an interesting question: for methods that are not first-order but use the gradient information only, could we be faster than the first-order methods?
Apparently Csawg is developed under a very weak assumption, which is the derivability of the loss function. 
In the studied convex function, we noted the algorithm surpasses the convergence rate of Nesterov's accelerated gradient, which may be orders faster than this theoretical rate limit for the first order methods. Answering this open question may help us understand gradient descent algorithms better, and explore new algorithms that use first-order information beyond the way that the first-order methods cover. 

Our work addresses the question in meta-learning asked by \citet{vilalta2002perspective}: {\em how can we exploit the knowledge of learning to improve the performance of learning algorithms?} Recent meta-gradient methods for meta-learning \citep{santoro2016meta,finn2017model,rusu2018meta,rajeswaran2019meta,yin2019meta,lee2019meta,khodak2019adaptive,javed2019meta,vanschoren2018meta,nichol2018first,finn2019online,hospedales2020meta} and reinforcement learning \citep{meta-gradient-rl,idbd_td,meta_trace,meta-gradient-questions} answered this question well by using gradient descent to learn task-specific parameters, task generalization parameters and algorithmic parameters, and the outcome of their research are systems that require less amount of data but are more competitive in generalization. However, all the experience used by these methods are the samples provided to the learning system, which are either from the real world, computer games or simulations. 
This paper uses another kind of experience, in the form of the gradient descent update, which is important because the knowledge of learning is best represented by the data generated in the learning process. Our results show that the gradient descent update data contains useful knowledge about the future parameters that appear in later iterations, which means step-size planning, a prediction based update method, can greatly accelerate the convergence of gradient descent. In this sense, our work is a validation of \citet{vilalta2002perspective}'s conjecture that learning algorithms can improve their performance hereby in the sense of learning speed through experience. Given the close relationship between step-size adaptation and meta-learning, we expect the update experience of SGD can be useful to improve the performance of meta-learning systems. In particular, it is interesting to explore whether the update experience on one task helps with the learning of another task.

\bibliographystyle{apalike}
\bibliography{ref}

\end{document}